\documentclass[sigconf]{acmart}
\AtBeginDocument{%
  \providecommand\BibTeX{{%
    \normalfont B\kern-0.5em{\scshape i\kern-0.25em b}\kern-0.8em\TeX}}}

\copyrightyear{2020}
\acmYear{2020}
\setcopyright{iw3c2w3}
\acmConference[WWW '20]{Proceedings of The Web Conference 2020}{April 20--24, 2020}{Taipei, Taiwan}
\acmBooktitle{Proceedings of The Web Conference 2020 (WWW '20), April 20--24, 2020, Taipei, Taiwan}
\acmPrice{}
\acmDOI{10.1145/3366423.3380148}
\acmISBN{978-1-4503-7023-3/20/04}
\usepackage{amssymb,amsfonts,bm}
\usepackage{enumerate}
\usepackage{mathrsfs}
\usepackage{hyperref}
\usepackage{tikz} 
\usepackage[caption=false,hangindent=10pt]{subfig}
\usetikzlibrary{shapes,positioning,fit,calc}
\usepackage[vlined, ruled, linesnumbered]{algorithm2e}
\newtheorem{thm}{\hspace{- 0.18 in} {\bf Theorem}}

\newtheorem{lem}[thm]{\hspace{- 0.18 in} {\bf Lemma}}

\theoremstyle{definition}

\theoremstyle{remark}

\setcopyright{none}
\renewcommand\footnotetextcopyrightpermission[1]{} 
\SetKwInput{KwInit}{Init}
\newcommand{\tr}{\operatorname{tr}}

\begin{document}
\title{Conversational Contextual Bandit: Algorithm and Application}

\author{Xiaoying Zhang}
\authornote{The work was done when the first author was an intern at Bytedance AI Lab.}
\affiliation{%
  \institution{CSE, The Chinese University of Hong Kong}
}
\email{xyzhang@cse.cuhk.edu.hk}

\author{Hong Xie}
\affiliation{%
  \institution{College of Computer Science, Chongqing University}
}
\email{xiehong2018@cqu.edu.cn}

\author{Hang Li}
\affiliation{%
  \institution{AI Lab, Bytedance}
  }
\email{lihang.lh@bytedance.com}

\author{John C.S. Lui}
\affiliation{%
  \institution{CSE, The Chinese University of Hong Kong}
}
\email{cslui@cse.cuhk.edu.hk}

\begin{abstract}
Contextual bandit algorithms provide principled online learning solutions to balance the exploitation-exploration trade-off in various applications such as recommender systems.
However, the learning speed of the traditional contextual bandit algorithms is often slow
 due to the need for extensive exploration.
This poses a critical issue in applications like recommender systems,
since users may need to provide feedbacks on a lot of uninterested items.
To accelerate the learning speed, we generalize contextual bandit to \textit{conversational contextual bandit}.
Conversational contextual bandit leverages not only behavioral feedbacks on arms (e.g., articles in news recommendation), but also occasional conversational feedbacks on key-terms from the user.
Here, a key-term can relate to a subset of arms, for example, a category of articles in news recommendation.
We then design the Conversational UCB algorithm (ConUCB) to address two challenges in conversational contextual bandit: 
(1) which key-terms to select to conduct conversation,
(2) how to leverage conversational feedbacks 
to accelerate the speed of bandit learning.
We theoretically prove that ConUCB can achieve a smaller regret upper bound 
than the traditional contextual bandit algorithm LinUCB,
which implies a faster learning speed.
Experiments on synthetic data, as well as real datasets from Yelp and Toutiao,
demonstrate
 the efficacy of the ConUCB algorithm. 
\end{abstract}


\maketitle
\fancyfoot{}
\thispagestyle{empty} 

\section{Introduction}

Contextual bandit serves as an invaluable tool for enhancing performance of a system through learning from interactions with the user while making trade-off between exploitation and exploration~\cite{abbasi2011improved,li2010contextual,li2016collaborative,wang2016learning}.
The contextual bandit algorithms have been applied to recommender systems, for
instance, to adaptively learn users' preference on items.
In this application,
the items are taken as the arms in contextual bandit, 
and the contextual vector of each arm/item contains
the observed information about the user and the item at the time.
The recommender system equipped with a contextual bandit algorithm sequentially recommends items to the user. 
The user provides a feedback (e.g., click) on the recommended item each round, which is viewed
as a reward.
The goal of the contextual bandit algorithm is to learn an item recommendation (arm
selection) strategy  to optimize the user's feedbacks in the long run
(cumulative rewards),  via utilizing the information of the user and items (contextual
vectors) as well as the user's feedbacks (rewards).  
In general, the algorithm needs to make a trade-off between exploitation (i.e.,
leveraging the user's preference already known) and exploration (i.e., revealing the user's
preference still unknown).

One shortcoming of the traditional contextual bandit algorithms~\cite{abbasi2011improved,li2010contextual,li2016collaborative,wang2016learning} lies in their slow learning speed. This is because they need to perform extensive exploration in order to collect sufficient feedbacks.
For applications like recommender systems, it poses a critical issue, because it
means that the user needs to provide feedbacks on a large number of items 
which she is not interested in.

Recently, a number of researchers propose the
construction of conversational recommender systems that leverage conversations to elicit
users' preference for better recommendation (e.g., ~\cite{christakopoulou2016towards,christakopoulou2018q}).
Inspired by this, we consider a novel contextual bandit setting in this paper, 
i.e., \textit{conversational contextual bandit}, which incorporates a conversation
mechanism into the traditional contextual bandit algorithm (as shown in
Figure~\ref{fig:ConBandit}), for accelerating bandit learning.
\begin{figure}[t]
  \centering
  \includegraphics[width=0.9\linewidth]{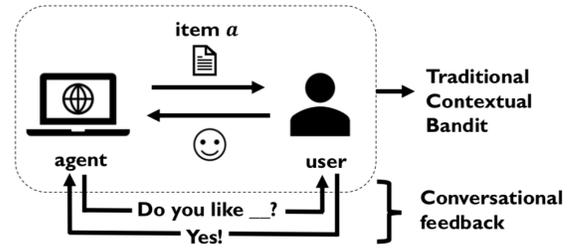}
  \caption{Conversational contextual bandit. The part in dashed box corresponds to traditional contextual bandit.}
  \label{fig:ConBandit}
  \vspace{-6mm}
\end{figure}

To illustrate the benefits of the conversation mechanism, let us consider the news
recommendation scenario with conversational contextual bandit.  
In this scenario, the agent/system also occasionally asks
questions with regard to the user's preference on key-terms.  
For example, asking about the user's preference on a category:
``Are you interested in news about {\it basketball}?'',
or asking about the user's preference on an entity: ``Do you like to read news
related to {\it
  LeBron James}?''.
There are two reasons why conversations can improve the learning speed.
First, the key-terms like ``basketball'' and ``LeBron James'' can be
associated with a large number of articles.  
Thus feedback on one key-term may contain a large amount of information about
the user's preference.  
Traditional contextual bandit algorithms~~\cite{abbasi2011improved,li2010contextual,li2016collaborative,wang2016learning} 
may spend many interactions to collect the information,
because many articles, that are related to the same key-term, may have different contextual vectors.  
For example, the contextual vector of an article about NBA games
may be far from that of  an article about basketball shoes, 
though they are all
related to the key-term ``basketball''.
Second, collecting explicit feedback from the user can help the system to capture the user's true preference faster.
For example, the fact that the user reads 
an article of NBA game with ``LeBron James'' 
may be because she concerns about 
the result of the game or because she is a fan of LeBron James.
To figure out which is more likely,
instead of recommending articles related to each possibility,
 a more convenient way is to directly
ask the user ``Do you want to read more news about \textit{LeBron James}?''.


In this paper, the agent conducts conversations by asking the user's
preference on key-terms.
Here, a key-term is related to a subset of arms, and can be a category or topic.
We assume that there exists a bipartite graph of key-terms and arms in which the relations between key-terms and arms are represented.
The agent occasionally selects key-terms 
and asks the user about her preference over the key-terms.  
Then, the user shows her feedbacks to the agent, 
for example, indicating whether she is interested in the key-terms.
The preference over the key-term is then ``propagated'' to the related arms,
and is leveraged by the algorithm to select arms and update the bandit model.
To the best of our knowledge, this is the first time such an approach is
proposed.

There are two main challenges for conversational contextual bandit:
\textit{(1) which key-terms to select for querying, (2) how to leverage conversational feedbacks
to build 
a more accurate bandit model.
}
We design the Conversational UCB algorithm (ConUCB) to address these challenges.

The ConUCB algorithm, as a generalization of
the LinUCB algorithm~\cite{abbasi2011improved}, repeatedly conducts learning as follows:
(1) If conversation is allowed at the round, given the current context and
historical interactions on arms and key-terms, ConUCB selects the key-terms that
reduce the learning error most, and enquires the user's preference over the
key-terms;
(2) It selects the arm 
with the largest upper confidence bound derived from
both arm-level and key-term-level feedbacks, and receives a reward. 
The interaction record will then be leveraged for key-term selection. 
Theoretical analysis shows 
that ConUCB achieves a lower cumulative regret upper bound 
than the standard contextual bandit algorithm LinUCB. 
Note that a smaller regret upper bound  means that the agent learns 
parameters more accurately within the same number of iterations, 
i.e., the speed of bandit learning is accelerated.
Experimental results on a synthetic dataset, as well as real datasets from Yelp
and Toutiao (the largest news recommendation platform in China) 
show that the ConUCB
algorithm  significantly outperforms 
the baseline algorithms including LinUCB.
Note that the basic mechanism of adding
conversation into bandit algorithms to speed up learning is generic.
We demonstrate that ConUCB can be easily extended to other
contextual bandit algorithms such as hLinUCB~\cite{wang2016learning}.

In summary, our contributions are as follows:
\begin{itemize}
\item We formulate the new conversational contextual bandit problem to
  improve the learning speed (Section~\ref{sec:model}).
\item We design the ConUCB algorithm by generalizing LinUCB for conversational contextual bandit and
  theoretically prove that it can achieve a smaller regret upper bound 
  than the conventional contextual bandit algorithm LinUCB, i.e., the learning
  speed is improved (Section~\ref{sec:algorithm}).
\item We empirically verify the improvement of ConUCB in learning speed  
using both
a synthetic dataset and real-world datasets (Section~\ref{sec:exp_syn} \& \ref{sec:real_data}).
\item We also extend ConUCB by adding the conversation mechanism into hLinUCB (Section~\ref{sec:discussion}).

\end{itemize}

\section{Problem Formulation}
\label{sec:model}
We first introduce traditional contextual bandit and then we generalize it to obtain conversational contextual bandit.  
Without loss of generality, we take the UCB algorithm as an example algorithm of contextual bandit.

\subsection{Contextual Bandit}
In contextual bandit, the agent learns to maximize the cumulative reward in 
$T \in \mathbb{N}_+$ rounds through interaction with the user.

Consider a finite set of $N$ arms denoted by $\mathcal{A}$.
At each round $t=1,\ldots, T$, the agent is given a subset of arms
$\mathcal{A}_t \subseteq \mathcal{A}$, and each arm $a \in \mathcal{A}_t$ is
associated with a $d$-dimensional contextual vector $\bm{x}_{a,t} \in
\mathbb{R}^d$, which describes the observable information of arm $a$ and the
user at round $t$. The agent chooses an arm $a_t$ on the basis of its contextual vector $\bm{x}_{a_t,t}$, shows it to the user, and receives a reward $r_{a_t,t} \in \mathcal{R}$. For example, $\mathcal{R}=\{0,1\}$ represents a binary reward, and $\mathcal{R}=\mathbb{R}$ represents a continuous reward.
The reward $r_{a_t,t}$ is a function of the contextual vector $\bm{x}_{a_t,t}$ and the parameter vector $\bm{\theta}$. The parameter vector $\bm{\theta}$ represents the user's preference and is what the agent wants to learn.
At each round $t$, the agent takes the selected arms $a_1, \ldots, a_{t-1}$ and
received rewards $r_{a_1, 1}, \ldots, r_{a_{t-1}, t-1}$ at the previous
rounds as input, to estimate the reward and select the arm $a_t \in \mathcal{A}_t$. 

The goal of the agent is to maximize the expected cumulative rewards.  
Let $\sum^T_{t=1}\mathbb{E}[r_{a^\ast_t,t}]$ denote the maximum expected
cumulative rewards in $T$ rounds, where $a^\ast_t \in \mathcal{A}_t$ is the
optimal arm at round $t$, i.e., $\mathbb{E}[r_{a^{\ast}_t,t}] \geq
\mathbb{E}[r_{a,t}], \forall a \in \mathcal{A}_t$. The goal of contextual bandit
learning is formally defined as minimization of the cumulative regret in $T$ rounds:
\begin{equation}
 \textstyle  R(T) \triangleq
  \sum\nolimits_{t=1}^T \left(
  \mathbb{E}[r_{a^*_t, t}]
  -
  \mathbb{E}[r_{a_t, t}]
  \right).
  \label{equ:regret}
\end{equation}
The agent needs to make a trade-off between exploitation 
(i.e., choose the best arm estimated from received feedbacks) 
and exploration (i.e., seek feedbacks from arms that the agent is unsure about). 

Let us consider the UCB (Upper Confidence Bound) algorithm.
The agent selects the arm $a_t$ at each round $t$ as follows:
\begin{equation}
\textstyle a_t = \arg\max_{a \in \mathcal{A}_t} 
 R_{a,t} + C_{a,t}, \nonumber
\end{equation}
where $R_{a,t}$ 
and $ C_{a,t}$ are the estimated reward and
confidence interval of arm $a$ at round $t$ respectively.  
The confidence interval $ C_{a,t}$ measures the uncertainty of 
reward estimation of arm $a$ at round $t$.
In the LinUCB algorithm~\cite{abbasi2011improved}, the reward function is defined as a linear function,
\begin{equation}
 \textstyle r_{a_t,t} = \bm{x}_{a_t,t}^T  \bm{\theta} + \epsilon_t,
  \label{equ:reward}
\end{equation}
where $\bm{\theta}$ is a $d$-dimensional parameter vector, $\epsilon_t$ is a random variable representing the random noise. One can also generalize LinUCB by utilizing non-linear reward~\cite{filippi2010parametric} or hidden features~\cite{wang2016learning}.

In news recommendation with contextual bandit, for example, the agent is the
recommender system, each arm $a$ is an article, the contextual vector
$\bm{x}_{a,t}$ denotes the observable information of the user and the article
$a$ at round $t$, and the reward $r_{a,t}$ is whether the user clicks article $a$ at round $t$. The parameter vector $\bm{\theta}$ represents the user's preference on articles. The goal of the system is to maximize the cumulative click-through rate (CTR) of the user.

\subsection{Conversational Contextual Bandit}
In conversational contextual bandit, the agent still learns to maximize the cumulative reward in $T$ rounds through interacting with the user. 
In addition to collecting feedbacks on selected arms, the agent also occasionally conducts conversations with the user and learns from conversational feedbacks.

We consider that the agent conducts conversations by asking the user's
preference on \textit{key-terms}, where a key-term is related to a subset of
arms, and can be a category, topic, etc. For example, in news recommendation, a
key-term can be a keyword, a key-phrase, or a combination of multiple single
keywords, extracted from articles. The agent may ask the user ``Are you
interested in news about {\it basketball}?''  
The user shows her answer to the question.

\noindent
\textbf{Problem formulation.} 
Consider a finite set of $N \in \mathbb{N}_+$ arms denoted by $\mathcal{A}$ 
and a finite set of $K \in \mathbb{N}_+$ key-terms denoted by $\mathcal{K}$.
The relationships between the arms and the key-terms are characterized by a
weighted bipartite graph $(\mathcal{A}, \mathcal{K}, \bm{W})$, 
whose nodes are divided into two sets $\mathcal{A}$ and $\mathcal{K}$, 
and weighted edges are represented by the matrix $\bm{W} \triangleq [w_{a,k}]$.   
Here $w_{a,k}$ represents the relationship between 
arm $a$ and key-term $k$.   
Without loss of generality, 
we normalize the total weights associated with each arm to be 1, 
i.e., $\sum_k w_{a,k}=1$.  

We also introduce a function $b(t)$ to model the frequency of conversations.
Note that $b(t)$ determines:
(1) whether to converse at round $t$;
(2) the number of conversations until round $t$.
To make it clearer, consider a function $q(t)$:
\begin{equation}
\textstyle q(t) = \left\{ \begin{array}{cl}
   1, & b(t)-b(t-1) > 0, \\
   0, & \mbox{otherwise}.  
\end{array} \right. \nonumber
\vspace{-2mm}
\end{equation}
If $q(t)=1$, then the agent conducts conversation with the user at round $t$;
if $q(t)=0$, then it does not.
The agent conducts $b(t)$ conversations up to round $t$.
For example, if $b(t)= k \lfloor\frac{t}{m}\rfloor, m \ge 1, k \ge 1$, then the agent makes
$k$ conversations in every $m$ rounds. If $b(t)= \lfloor \log(t) \rfloor$, then
the agent makes a conversation with a frequency represented by the logarithmic
function of $t$. If $b(t)\equiv 0$, then there is no conversation between the
agent and the user.
Moreover, we assume that key-term-level conversations should 
be less frequent than arm-level interactions, i.e., $b(t)\leq t, \forall t$, 
in consideration of users' experience.

At each round $t=1,\ldots, T$, the agent is given a subset of arms
$\mathcal{A}_t \subseteq \mathcal{A}$, and each arm $a \in \mathcal{A}_t$ is
specified by a contextual vector $\bm{x}_{a,t} \in \mathbb{R}^d$.
 Without loss of generality, we normalize the contextual vector 
such that $\|\bm{x}_{a,t}\|_2=1$.
Based on the arm-level feedbacks and conversational feedbacks received in the
previous $t-1$ rounds, 
\begin{itemize}
\item If $q(t)=1$, the agent conducts $\lfloor b(t)- b(t-1) \rfloor$ conversations with the user.
  In each conversation, the agent asks the user's preference on one selected
  key-term $k \in \mathcal{K}$, and gets the user's feedback $\tilde{r}_{k, t}$.
  For example, $\tilde{r}_{k,t}$ can be binary $\{0, 1\}$,  where $0$ and $1$ stand for negative and positive feedback respectively.
\item The agent chooses an arm $a_t \in \mathcal{A}_t$,
  presents it to the user,  and receives the reward of arm $a_t$, denoted by $r_{a_t,t}$.
\end{itemize}

The agent tries to speed up the learning process by leveraging conversational feedbacks. The problems in conversational contextual bandit are to find 
(a) an effective arm selection strategy and 
(b) an effective key-term selection strategy, so that after $T$ rounds of arm-level interactions and $b(T)$ times of conversations, the cumulative regret in Eq.~(\ref{equ:regret}) is significantly reduced.

\noindent
\textbf{General algorithm of conversational UCB (ConUCB).} Algorithm
\ref{alg:CCB} outlines the general algorithm of ConUCB,
which will be described in detail in Section~\ref{sec:algorithm}.
 One can see that when no conversation is performed between the agent and the
 user, conversational contextual bandit degenerates to standard contextual
 bandit. In ConUCB, the agent selects an arm with the following strategy
\begin{equation}
  \textstyle a_t = \arg\max\nolimits_{a \in \mathcal{A}_t} \tilde{R}_{a,t}+ C_{a,t},
  \label{equ:ConUCB_o}
\end{equation}
where $\tilde{R}_{a,t}$ and $C_{a,t}$ are the estimated reward and the
confidence interval of arm $a$ at round $t$ respectively.
As will be shown in Section~\ref{sec:algorithm} later, $ \tilde{R}_{a,t}$ is inferred from \textit{both} arm-level and key-term-level
feedbacks, and $C_{a,t}$ contains \textit{both} arm-level and key-term-level confidence interval. 

\begin{algorithm}[t]
  \caption{{\bf General algorithm of ConUCB}}
   \label{alg:CCB}
	\KwIn{arms $\mathcal{A}$, key-terms $\mathcal{K}$, graph $(\mathcal{A}, \mathcal{K}, \bm{W})$, $b(t)$. }
  \For{ $t=1,\ldots, T$}
      {
        observe contextual vector $\bm{x}_{a,t}$ of each arm $a \in \mathcal{A}_t$\;
        If conversation is allowed at round $t$, i.e., $q(t)=1$, select key-terms to conduct conversations and receive conversational feedbacks $\{\tilde{r}_{k,t}\}$\;
        select an arm $a_t = \arg\max_{a \in \mathcal{A}_t} \tilde{R}_{a,t} + C_{a,t}$ \;
        receive a reward $r_{a_t,t}$\;
        update model \;
      }
      \vspace{-2mm}
\end{algorithm}
\setlength{\textfloatsep}{0pt}

In ConUCB,  
the user's feedback on key-term $k$, i.e., $\tilde{r}_{k,t}$, 
is estimated from the user's feedbacks on related arms, i.e.,
\begin{equation}
 \textstyle \mathbb{E}[\tilde{r}_{k,t}]
  = \sum\nolimits_{a\in \mathcal{A}} 
  \frac{w_{a,k}}{\sum_{a' \in \mathcal{A}} w_{a',k}} 
  \mathbb{E}[r_{a,t}],\quad k \in \mathcal{K}.
  \label{con_feedback_model}
\end{equation}
Equivalently, 
$\tilde{r}_{k,t} 
= \sum_{a\in \mathcal{A}}
  \frac{w_{a,k}}{\sum_{a' \in \mathcal{A}} w_{a',k}}
  \mathbb{E}[r_{a,t}]
   +\tilde{\epsilon}_{k,t}$,
   where $\tilde{\epsilon}_{k,t}$ is a random variable representing the random noise in reward.
When both the conversational feedback $\tilde{r}_{k,t}$ and arm-level feedback $r_{a,t}$ are binary, Eq.~(\ref{con_feedback_model}) has a probabilistic interpretation. Specifically, for binary feedback $c \in \{0, 1\}$, we have
\begin{equation}\label{con_feedback_model2}
 \textstyle \mathbb{P}(c|k,t)=\sum\nolimits_{a \in \mathcal{A}} \mathbb{P}(a|k)\mathbb{P}(c|a,t),
\end{equation}
where $\mathbb{P}(c=1|k,t)$ and $\mathbb{P}(c=0|k,t)$ represent the
probabilities that the user gives a positive and negative feedback to the
question with key-term $k$ at round $t$ respectively;
$\mathbb{P}(c=1|a,t)$ and $\mathbb{P}(c=0|a,t)$ denote the probabilities that
the user likes and dislikes arm $a$ at round $t$ respectively. Thus, if we take
$\mathbb{E}[r_{a,t}]= \mathbb{P}(c=1| a,t)$,
$\mathbb{E}[\tilde{r}_{k,t}]=\mathbb{P}(c {=} 1 | k,t)$, and $\mathbb{P}(a|k) {=}
\frac{w_{a,k}}{\sum_{a' \in \mathcal{A}} w_{a',k}}$, we obtain Eq.
(\ref{con_feedback_model2}).


\section{Algorithm \& Theoretical Analysis}
\label{sec:algorithm}
In this section, we present the details of ConUCB by providing specific solutions to the two problems in
Algorithm~\ref{alg:CCB}: (1) how to select key-terms to conduct conversation
(Line 3), (2) how to select an arm, i.e.,
calculate $\tilde{R}_{a,t}$ and $C_{a,t}$ (Line 4).
ConUCB is a generalization of LinUCB~\cite{abbasi2011improved} in the sense that
the arm-level reward
function is the same as that of LinUCB in Eq.~(\ref{equ:reward}).
We theoretically analyze the upper bound of its cumulative regret and discuss the
impact of conversational feedbacks.
We note that ConUCB has a generic mechanism of 
selecting key-terms and leveraging feedbacks on key-terms to speed up learning, which can be easily incorporated into a variety of contextual bandit algorithms  such as CoFineUCB~\cite{yue2012hierarchical},
  hLinUCB~\cite{wang2016learning}.  
In Section~\ref{sec:discussion}, we explain how to 
apply the same technique to the hLinUCB algorithm.

\subsection{ConUCB Algorithm}

 The ConUCB algorithm is described in 
 Algo.~\ref{alg:ConUCB_all}.
 It contains a key-term selection module to select
 key-terms (line 2-11) and an arm-selection module to select arms (line 12-15).
 The two modules collaborate with each other as follows:
\begin{itemize}
  \item If conversation is allowed at round $t$, given the current context and interaction histories on both arms and key-terms, the key-term selection module repeatedly selects the key-term that minimizes the regret and asks the user's preference over it (line 5). 
  Then the newly estimated key-term-level parameter vector
    ($\tilde{\bm{\theta}}_t$) is passed to the arm-selection module.
  \item  Under the guidance of $\tilde{\bm{\theta}}_t$ and rewards received up to round $t$, the arm-selection module recommends an arm to the user,
    and receives a reward (line 12-15).
    The interaction record will then be
    leveraged by 
    the key-term selection module. 
  \end{itemize}

\subsubsection{\textbf{Arm selection.}}
At round $t$, ConUCB first  estimates the user's preference at the key-term level, denoted as
$\tilde{\bm{\theta}}_t$, by

\resizebox{\columnwidth}{!}{
   \begin{minipage}{\columnwidth}
\begin{align}
     \textstyle \boldsymbol{\tilde{\theta}}_t 
      {=} \arg\min_{\boldsymbol{\tilde{\theta}}} 
\sum^t_{\tau =1} 
\sum_{k \in \mathcal{K}_\tau }
\left( \frac{\sum_{a \in \mathcal{A}} w_{a,k} \boldsymbol{\tilde{\theta}}^T \boldsymbol{x}_{a,\tau}}{\sum_{a \in \mathcal{A}} w_{a,k}} {-} \tilde{r}_{k,\tau} \right)^2 {+} \tilde{\lambda} \| \boldsymbol{\tilde{\theta}}\|_2^2, \nonumber 
    \label{equ:tilde_theta}
\end{align}
\vspace{1mm}
\end{minipage}}
where $\mathcal{K}_\tau$ denotes the set of key-terms queried at round $\tau$. We set $\mathcal{K}_\tau = \emptyset$ if no key-term is queried,
and we let $\mathcal{K}_\tau$ contain duplicated elements if querying a
key-term multiple times at the round. The coefficient $\tilde{\lambda} \in
\mathbb{R} $ controls regularization.
Then $
\boldsymbol{\tilde{\theta}}_t$ is used to guide the learning of the arm-level parameter
vector at round $t$:
   \begin{equation}
     \begin{aligned}
      \textstyle \boldsymbol{\theta_t} 
       {=} \arg\min_{\boldsymbol{\theta}} \ \lambda \sum_{\tau=1}^{t-1}(\boldsymbol{\theta}^T\boldsymbol{x}_{a_{\tau},\tau} {-} r_{a_{\tau},\tau})^2 {+}(1 {-} \lambda)\|\boldsymbol{\theta}-\boldsymbol{\tilde{\theta}}_t\|_2^2.  \nonumber
     \end{aligned}
     \label{equ:theta}
   \end{equation}
where $\lambda {\in} [0,1]$ balances learning from rewards at arm-level and learning from feedbacks at key-term-level (i.e., $\bm{\tilde{\theta}}_t$).
Both optimization problems have closed-form solutions,
   $\boldsymbol{\tilde{\theta}}_t=\boldsymbol{\tilde{M}}_t^{-1} \boldsymbol{\tilde{b}}_t$
   and $\boldsymbol{\theta}_t=\boldsymbol{M}_t^{-1}(\boldsymbol{b}_t+(1-\lambda) \boldsymbol{\tilde{\theta}}_t)$, where
   \begin{align}
       &\textstyle \boldsymbol{\tilde{M}}_t 
      {=} \sum_{\tau=1}^t \sum_{k \in \mathcal{K}_{\tau}} 
      \!\!\left( \frac{\sum_{a \in \mathcal{A}} w_{a,k}\boldsymbol{x}_{a,\tau}}{\sum_{a \in \mathcal{A}} w_{a,k}} \right) \!\! \left( \frac{\sum_{a \in \mathcal{A}} w_{a,k}\boldsymbol{x}_{a,\tau}}{\sum_{a \in \mathcal{A}} w_{a,k}} \right)^T \!{+}\tilde{\lambda}\boldsymbol{I}, \nonumber \\
       &\textstyle \boldsymbol{\tilde{b}}_t=\sum\nolimits_{\tau=1}^t \sum\nolimits_{k \in \mathcal{K}_{\tau}}  \left( \frac{\sum_{a \in \mathcal{A}} w_{a,k}\boldsymbol{x}_{a,\tau}}{\sum_{a \in \mathcal{A}} w_{a,k}} \right)  \tilde{r}_{k,\tau}, \nonumber\\
       &\textstyle \boldsymbol{M}_t=\lambda
       \sum\nolimits_{\tau=1}^{t-1} \boldsymbol{x}_{a_\tau,\tau} 
      \boldsymbol{x}_{a_\tau,\tau}^T+(1-\lambda)\boldsymbol{I},  
      \nonumber
      \\
     &   \textstyle \boldsymbol{b}_t=\lambda \sum\nolimits_{\tau=1}^{t-1} 
      \boldsymbol{x}_{a_\tau,\tau}r_{a_\tau,\tau} .
     \label{equ:closed_solution}
   \end{align}

 To apply the arm-selection strategy in Eq.~(\ref{equ:ConUCB_o}) , we
  need to derive the confidence interval
 $C_{a,t}$.
 Based on the closed-form solutions of $\bm{\theta}_t$ and $\bm{\tilde{\theta}}_t$ in Eq.~(\ref{equ:closed_solution}), we can prove that Lemma~\ref{lem:conf_int} holds, and thus 
 \begin{align}
 \textstyle C_{a,t}= \lambda\alpha_t
      \|\boldsymbol{x}_{a,t}\|_{\boldsymbol{M}_t^{-1}} + (1-\lambda) \tilde{\alpha}_t
       \|\boldsymbol{x}_{a,t}^T\boldsymbol{M}_t^{-1}\|_{\boldsymbol{\tilde{M}}_t^{-1}}, \nonumber
\end{align}
 where $\alpha_t$ is defined in Lemma1, and $\tilde{\alpha}_t$ represents the estimation error of $\bm{\tilde{\theta}}_t$ and is determined by how the agent
 selects key-terms.
 We will discuss $\tilde{\alpha}_t$ later in
  Section~\ref{sec:key-term-selection}, and show that $\tilde{\alpha}_t < \alpha_t$.

Consequently, ConUCB selects an arm according to the following strategy (Eq.~(\ref{equ:ConUCB_o})):
  \begin{align}
    \textstyle a_t=\arg \max_{a \in \mathcal{A}_t} \underbrace{\bm{x}_{a,t}^T\bm{\theta}_t}_{\tilde{R}_{a,t}}+\underbrace{  \lambda\alpha_t
     \|\boldsymbol{x}_{a,t}\|_{\boldsymbol{M}_t^{-1}} + (1-\lambda) \tilde{\alpha}_t
    \|\boldsymbol{x}_{a,t}^T\boldsymbol{M}_t^{-1}\|_{\boldsymbol{\tilde{M}}_t^{-1}} }_{C_{a,t}}.
    \label{equ:arm_selection}
    \vspace{-8mm}
  \end{align}
In Eq.~(\ref{equ:arm_selection}), $\tilde{R}_{a,t}$ is the estimated reward of arm
$a$ at round $t$, based on the current estimated parameter vectors at arm-level and
key-term level.
It represents exploitation of currently promising arms.
$C_{a,t}$ denotes the uncertainty in reward estimation of arm $a$, which contains two parts: (1) uncertainty from the noise in arm rewards received until round $t$ (the first term) ; (2) uncertainty from the estimated key-term-level parameter vector $\tilde{\bm{\theta}}_t$ (the
second term).
It represents exploration of currently less promising arms.
It is easy to verify that $C_{a,t}$ shrinks when more interactions between the agent and user are carried out, and thus the exploitation and exploration are balanced.
Moreover, one can see that the second term of $C_{a,t}$ shrinks more
quickly than the first term (more details can be fund at  Eq.~(\ref{appendix_equ:x_M_M_M_x}) in Appendix), indicating the benefit of conversation.

   \begin{algorithm}[t]
\caption{{\bf ConUCB algorithm}}
  \label{alg:ConUCB_all}
	\KwIn{
	graph $(\mathcal{A},\mathcal{K},\bm{W})$,
  conversation frequency function $b(t)$.}
  \KwInit{ $\boldsymbol{\tilde{M}}_{0}= \tilde{\lambda}\boldsymbol{I}$,
    $\boldsymbol{\tilde{b}}_0=\boldsymbol{0}$, $\boldsymbol{M}_0=(1-\lambda)\boldsymbol{I}$, $\boldsymbol{b}_0=\bm{0}$. }
	\For{$t=1,2,...,T$}{
		\If{$b(t)-b(t-1)>0$}
		  {
        nq= $b(t)-b(t-1)$\;
       \While{nq$>0$}
       {
           Select a key-term $k \in \mathcal{K}$ according to
           Eq.~(\ref{equ:optimized_strategy}), and query the user's preference
           over it \;
           Receive the user's feedback $\tilde{r}_{k,t}$\;
    $\boldsymbol{\tilde{M}}_t=  \boldsymbol{\tilde{M}}_{t-1} +  \left(
      \frac{\sum_{a \in \mathcal{A}}w_{a,k}\boldsymbol{x}_{a,t}}{\sum_{a \in \mathcal{A}} w_{a,k}} \right)
    \left( \frac{\sum_{a \in \mathcal{A}} w_{a,k}\boldsymbol{x}_{a,t}}{\sum_{a \in \mathcal{A}} w_{a,k}}
    \right)^T $\;
    $\boldsymbol{\tilde{b}}_t= \boldsymbol{\tilde{b}}_{t-1} + \left( \frac{\sum_{a \in \mathcal{A}} w_{a,k}\boldsymbol{x}_{a,t}}{\sum_{a \in \mathcal{A}} w_{a,k}} \right)  \tilde{r}_{k,t} \nonumber$ \;
    nq-=1
    }
    }
    \Else{
      $\boldsymbol{\tilde{M}}_t=  \boldsymbol{\tilde{M}}_{t-1}$,
      $\boldsymbol{\tilde{b}}_t= \boldsymbol{\tilde{b}}_{t-1}$ \;
      }
  $\boldsymbol{\tilde{\theta}}_t=\boldsymbol{\tilde{M}}_t^{-1}
  \boldsymbol{\tilde{b}}_t$,
  $\boldsymbol{\theta}_t=\boldsymbol{M}_t^{-1} \left( \boldsymbol{b}_t +
    (1-\lambda) \boldsymbol{\tilde{\theta}}_t  \right)$ \;
  Select $a_t=\arg \max_{a\in \mathcal{A}_t} \boldsymbol{x}_{a,t}^T\boldsymbol{\theta}_t+ \lambda\alpha_t
      \|\boldsymbol{x}_{a,t}\|_{\boldsymbol{M}_t^{-1}} + (1-\lambda) \tilde{\alpha}_t
       \|\boldsymbol{x}_{a,t}^T\boldsymbol{M}_t^{-1}\|_{\boldsymbol{\tilde{M}}_t^{-1}}$\;
  Ask the user's preference on arm $a_t \in \mathcal{A}$ and receive the reward $r_{a_t,t}$ \;
  $\boldsymbol{M}_t= \boldsymbol{M}_t +\lambda \boldsymbol{x}_{a_t,t}\boldsymbol{x}_{a_t,t}^T$,  
\hspace{0.18 in}  
    $\boldsymbol{b}_t=  \boldsymbol{b}_t + \lambda \boldsymbol{x}_{a_t,t}r_{a_t,t} $ \;
	}
\vspace{-2mm}
\end{algorithm} 
\setlength{\textfloatsep}{0pt}

\begin{lem}
   Let  $\bm{\theta}_*$ and
 $\tilde{\bm{\theta}}_*$ denote the unknown true parameter
 vectors of the user at arm level and key-term level respectively.
  Assume that $\|\boldsymbol{\tilde{\theta}}_t-\boldsymbol{\tilde{\theta}}_*\|_{\boldsymbol{\tilde{M}}_t}\leq
   \tilde{\alpha}_t$, $\tilde{\epsilon}_{k,t}$ and $\epsilon_t$
  are conditionally $1$-sub-Gaussian, then for $\forall t, a\in \mathcal{A}$, with probability $1-\sigma$, the following inequality holds
    \begin{align}
    \textstyle 
    & |\boldsymbol{x}_{a,t}^T(\boldsymbol{\theta}_t -\boldsymbol{\theta}_*)|
    \leq 
     \lambda\alpha_t
      \|\boldsymbol{x}_{a,t}\|_{\boldsymbol{M}_t^{-1}} + (1-\lambda) \tilde{\alpha}_t
       \|\boldsymbol{x}_{a,t}^T\boldsymbol{M}_t^{-1}\|_{\boldsymbol{\tilde{M}}_t^{-1}}, \nonumber
      \end{align}
   where $\alpha_t= \sqrt{d \log\left( (1+\frac{\lambda
         t}{(1-\lambda)d})/\sigma \right)}$, and  $\|\bm{x}\|_{\bm{M}}=\sqrt{\bm{x}^T \bm{M} \bm{x}}$.
  \label{lem:conf_int}
 \end{lem}

\subsubsection{\textbf{Key-Term Selection}}
\label{sec:key-term-selection}

Next we describe how the ConUCB algorithm selects key-terms.
Let $\boldsymbol{X}_t \in \mathbb{R}^{|\mathcal{A}_t| \times d}$
 denote the collection of contextual vectors of arms presented at round $t$,  i.e,
 $\boldsymbol{X}_t(a,:)=\boldsymbol{x}_{a,t}^T, \forall a \in \mathcal{A}_t$.
Given the current context
$X_t$ and  interactions on both arms and key-terms up to round $t$, 
to minimize the cumulative regret, ConUCB needs to select a key-term so that $\bm{\theta}_t$ can be learned as accurately as possible,
since no regret would be induced if $\bm{\theta}_t$ is equal to the unknown true arm-level parameter vector $\bm{\theta}_*$.
Thus, a natural idea is to select the key-term that minimizes the estimation error
 $\mathbb{E} [ \|\boldsymbol{X}_t\boldsymbol{\theta}_t
 -\boldsymbol{X}_t\boldsymbol{\theta}_*\|_2^2 ]$.
 As suggested by Theorem~\ref{thm:key-term-selection}, this means to select key-terms according to Eq.~(\ref{equ:optimized_strategy}).

 \begin{thm}
    Given the current context $X_t$ and interactions at both arm-level and key-term level up to
    round $t$, to minimize $\mathbb{E} [ \|\boldsymbol{X}_t\boldsymbol{\theta}_t
 -\boldsymbol{X}_t\boldsymbol{\theta}_*\|_2^2 ]$, one only needs to select the key-term $k$ to minimize
\begin{align}
 \tr\left(
     \boldsymbol{X}_t\boldsymbol{M}_t^{-1}(\boldsymbol{\tilde{M}}_{t-1}+\boldsymbol{\tilde{x}}_{k,t}\boldsymbol{\tilde{x}}_{k,t}^T)^{-1}\boldsymbol{M}_t^{-1}\boldsymbol{X}_t^T
    \right). \nonumber
\end{align} 
 In other words, it selects the key-term $k$ as follows:
 \begin{align}
  k&= \arg\max_{k'}  \quad 
  \| \boldsymbol{X}_t\boldsymbol{M}_t^{-1}\boldsymbol{\tilde{M}}_{t-1}^{-1}\boldsymbol{\tilde{x}}_{k',t}\|_2^2 \big/
  \left(1+\boldsymbol{\tilde{x}}_{k',t}^T\boldsymbol{\tilde{M}}_{t-1}^{-1}\boldsymbol{\tilde{x}}_{k',t} \right).
  \label{equ:optimized_strategy}
 \end{align}
 where $\boldsymbol{\tilde{x}}_{k,t}=\sum_{a \in \mathcal{A}} \frac{w_{a,k}} 
 {\sum_{a' \in \mathcal{A}} w_{a',k}}\boldsymbol{x}_{a,t}$.
 \label{thm:key-term-selection}
\end{thm}

We can observe from Eq.~(\ref{equ:optimized_strategy}) that the key-term selection is depended on both the arms and key-terms selected in the previous rounds, i.e., $\bm{M}_t$ and $\bm{\tilde{M}}_t$. 
Moreover, the essence of Theorem~\ref{thm:key-term-selection} is to select key-terms to minimize $ \tr\left(
     \boldsymbol{X}_t\boldsymbol{M}_t^{-1}\boldsymbol{\tilde{M}}_{t}^{-1}\boldsymbol{M}_t^{-1}\boldsymbol{X}_t^T
    \right)$,
which can make the last term in the arm selection strategy, i.e.,  $\tilde{\alpha}_t\|\boldsymbol{x}_{a,t}^T\boldsymbol{M}_t^{-1}\|_{\boldsymbol{\tilde{M}}_t^{-1}}
    $ in Eq.~(\ref{equ:arm_selection}), shrink faster.

  When selecting key-terms
according to  Eq.~(\ref{equ:optimized_strategy}), 
the agent can calculate $\tilde{\alpha}_t$ using 
Lemma~\ref{lem:fixed_theta_bound}.
Since $b(t) \leq t$,  $\tilde{\alpha}_t$ is at the order of $O(\sqrt{d+\log t})$, while $\alpha_t$ is at the order of $O(\sqrt{d \log t})$, implying $\tilde{\alpha}_t <\alpha_t $.
  
\begin{lem}
   In selection of key-terms according to Eq.~(\ref{equ:optimized_strategy}), with
   probability $1-\sigma$, the following inequality holds:
   
   \resizebox{0.95\columnwidth}{!}{
\begin{minipage}{\columnwidth}
      \begin{align}
\textstyle
        \|\boldsymbol{\tilde{\theta}}_t-\boldsymbol{\tilde{\theta}}_*\|_{\boldsymbol{\tilde{M}}_t} \leq   \tilde{\alpha}_t = \sqrt{2 \left (d\log6+\log(\frac{2b(t)}{\sigma}) \right)} +   2\sqrt{\tilde{\lambda}}\|\boldsymbol{\tilde{\theta}}_*\|_2. \nonumber
      \end{align}
    \end{minipage}}
      \label{lem:fixed_theta_bound}
    \end{lem}

 \subsection{{\bf Regret Upper Bound of ConUCB}} 
 \label{sec:regret_bound_conucb}

We can prove that the regret upper bound of ConUCB is as follows.
\begin{thm}
  Assume that $\lambda\in
  [0,0.5]$, $\tilde{\lambda} \geq
  \frac{2(1-\lambda)}{\lambda(1-\sqrt{\lambda})^2}$, with the key-term selection strategy defined
    in Eq.~(\ref{equ:optimized_strategy}), then with probability $1-\sigma$, ConUCB  has the following regret upper bound.

     \resizebox{0.95\columnwidth}{!}{
\begin{minipage}{\columnwidth}
  \begin{align*}
\textstyle
      R(T)
      &  \textstyle \leq
      2\left(  \sqrt{\lambda} \sqrt{d\log\left( (1+\frac{\lambda T}{(1-\lambda)d})/\sigma \right)}
      + 2\sqrt{\frac{1-\lambda}{\lambda}}\|\boldsymbol{\tilde{\theta}}_*\|_2
      \right.
      \\
\textstyle
      & \textstyle
      \left.
      +  (1-\sqrt{\lambda})\sqrt{ d\log  6+\log(\frac{2b(t)}{\sigma}) }
      \right)
      \sqrt{Td \log(1+\frac{\lambda T}{d(1-\lambda)})}.
  \end{align*}
\end{minipage}}
 \label{thm:regret_bound_fixed_scenario}
\end{thm}

Since $b(t) \leq t, \lambda \in [0,0.5]$, the regret upper bound of ConUCB in
Theorem~\ref{thm:regret_bound_fixed_scenario}  is at most
$\dot{O}((1-\sqrt{\lambda})\sqrt{d+\log T} + \sqrt{\lambda d\log T})$, where $\dot{O}(a)=O(a \cdot \sqrt{dT\log T})  $. 
It is smaller than
the regret upper bound of LinUCB~\cite{abbasi2011improved}, i.e., $\dot{O}((1-\sqrt{\lambda})\sqrt{d\log T} + \sqrt{\lambda d\log T})=\dot{O}(\sqrt{d \log T})$. 
Therefore, when $d$ and $\log T$ are large, which is usually the case in practice, we can improve substantially by reducing $\sqrt{d \log T}$ to $\sqrt{d+\log T}$ in the first term.


\section{Experiments on Synthetic Data}
\label{sec:exp_syn}
In this section, we describe experimental results 
on synthetic data.

\subsection{Experiment Setting}
\label{sub_sec:syn_setting}

\noindent
{\bf Synthetic data.}
We create a key-term set $\mathcal{K} \triangleq \{1, 2, \ldots, K\}$ with size
$K$ and an arm pool (set) $\mathcal{A} \triangleq \{a_1, a_2, \ldots, a_N\}$
with size $N$.
Each arm $a \in \mathcal{A}$ is associated with a $d$-dimensional feature vector $\bm{x}_a$,
and it is also related to a set of key-terms $\mathcal{Y}_a \subseteq \mathcal{K}$ 
with equal weight $1 / |\mathcal{Y}_a|$.  
We assume that the more shared key-terms two arms have, 
the closer their feature vectors will be.  
We generate an arm's feature vector as follows:
(1) we first generate a pseudo feature vector $\bm{\dot{x}}_k$
for each key-term $k \in \mathcal{K}$, 
where each dimension of $\bm{\dot{x}}_k$ is drawn 
independently from a uniform distribution $U(-1,1)$; 
(2) for each arm $a$, we sample $n_a$ key-terms uniformly at random  
from $\mathcal{K}$ without replacement as its related key-terms set $\mathcal{Y}_a$
with equal weight $1 / n_a$, where $n_a$ is a random number in $[1,M]$;
(3) finally, each dimension $i$ of $\bm{x}_a$ is 
independently drawn from 
$
N(\sum_{k \in \mathcal{Y}_a} \bm{\dot{x}}_k(i) / n_a, \sigma_g^2)
$, 
where $\sigma_g \in \mathbb{R}_+$.
We note that the information of key-terms
is contained in $\bm{x}_a$, and thus they are available for all algorithms. 
We then generate $N_u$ users,  
each of whom is associated with a $d$-dimensional vector $\bm{\theta}_u$, 
i.e., the ground-truth of user $u$'s preference.  
Each dimension of $\bm{\theta}_u$ is drawn from a uniform distribution U(-1,1).
We consider a general setting in which 
at each round $t$, the simulator only discloses a subset
of arms in $\mathcal{A}$, denoted as $\mathcal{A}_t$, to the agent for selection, 
for example, randomly selecting 50 arms from $\mathcal{A}$ without replacement.
The true arm-level reward $r_{a,t}$ as well as  key-term level reward $\tilde{r}_{k,t}$ are generated according to Eq.~(\ref{equ:reward}) and Eq.~(\ref{con_feedback_model}) respectively.  The noise $\epsilon_t$ is sampled from  Gaussian distribution
$\mathcal{N}(0,\sigma_g^2)$ only once for all arms
at round $t$, so is the noise $\tilde{\epsilon}_{k,t}$.
In the simulation, we set contextual vector $d=50$, user number $N_u=200$,
arm size $N=5000$, size of key-terms $K=500$,  $\sigma_g=0.1$.
We set $M=5$, which means that each arm is related to at most $5$ different key-terms, 
because the average values of $M$ in Yelp and Toutiao are 4.47 and 4.49 respectively.
Moreover, following similar procedures in paper~\cite{li2016collaborative}, we tune
the optimal parameters in each algorithm. 

\noindent
{\bf Baselines.}
We compare the proposed ConUCB algorithm with the following algorithms. 
\begin{itemize}
\item
  {\bf LinUCB}~\cite{li2010contextual}:  The state-of-art contextual bandit algorithm.
  LinUCB only works with arm-level feedback, 
  and it does not consider conversational feedback.
\item
  {\bf Arm-Con}:
 Christoper et. al.~\cite{christakopoulou2016towards} proposes to conduct
 conversation  by asking the user whether she likes an additional
 item selected by a bandit algorithm. Arm-Con adopts the conversation format and
 leverages LinUCB for arm selection.
\item
{\bf Var-RS}: A variant of ConUCB selecting a key-term \textit{randomly}.

\item
  {\bf Var-MRC}:   A variant of ConUCB that selects the key-term with
  the maximal related confidence under current context: 
 \[
    \textstyle \qquad \qquad  k=\arg\max_{k'} \sum_{a \in \mathcal{A}_t} \frac{w_{a,k'}}{\sum_{a' \in \mathcal{A}_t} 
     w_{a',k'}} \tilde{\alpha}_t
    \|\boldsymbol{x}_{a,t}^T\boldsymbol{M}_t^{-1}\|_{\boldsymbol{\tilde{M}}_t^{-1}},
    \vspace{-2mm}
 \]
 where $\tilde{\alpha}_t
    \|\boldsymbol{x}_{a,t}^T\boldsymbol{M}_t^{-1}\|_{\boldsymbol{\tilde{M}}_t^{-1}}$ is the part of confidence interval $C_{a,t}$ related to key-terms (Eq.~(\ref{equ:arm_selection})).

\item
{\bf Var-LCR}:  A variant of ConUCB  that selects the key-term with
   the largest confidence reduction under current context: 
 \[
    \textstyle \qquad k= \arg\max\limits_{k'} \sum_{a \in \mathcal{A}_t} \frac{w_{a,k'}}
{\sum_{a' \in \mathcal{A}_t} w_{a',k'}}\left(C_{a,t} -C_{a,t}^{k'} \right), 
\vspace{-2mm}
 \]
      where $C_{a,t}^{k'}$ is the new confidence interval of arm $a$ at round $t$, if we query key-term $k'$ next. 


\end{itemize}
Note that if we run LinUCB algorithm $T$ rounds,
then all other algorithms conduct $T$-round arm-level
interactions and $b(T)$ conversations.
Conversations are conducted at the same time for the algorithms. 

\subsection{Evaluation Results}
\label{sub_section:syn_regret}
We evaluate all algorithms in terms of cumulative regret defined in Eq. (\ref{equ:regret}).
We set the frequency function $b(t)=5\lfloor \log(t) \rfloor$.
At each round $t$, we randomly select 50
arms from $\mathcal{A}$ without replacement as $\mathcal{A}_t$.
The same $\mathcal{A}_t$ are presented to all 
algorithms.

\noindent
{\bf Cumulative regret comparison.} 
We run the experiments 10 times, and calculate the average cumulative regret for each algorithm. 
The results plotted in Figure~\ref{fig:syn_regret} show that LinUCB has the largest cumulative regret, indicating
that the use of conversational feedbacks,
either by querying items (Arm-Con) or by querying key-terms
(Var-RS, ConUCB, Var-MRC, Var-LCR)  
can improve the learning speed of bandit,
since they can learn 
parameters more accurately within the same number of iterations.
Moreover, 
 algorithms that query key-terms, i.e., Var-RS, ConUCB, Var-MRC, Var-LCR,
have much smaller cumulative regret than Arm-Con, 
which asks the user's preference on additional items.  
It is reasonable,
because feedback on a key-term should be more informative than feedback on an arm.
Finally,
the ConUCB algorithm has the smallest cumulative regret, demonstrating the effectiveness of its key-term selection strategy.

\noindent
{\bf Accuracy of learned parameters.}
We next compare the accuracy of the learned
parameters in different algorithms.
For each algorithm, we calculate the average difference between the learned parameter vectors of users and
the ground-truth parameter vectors of users, i.e.,  $\frac{1}{N_u} \sum_u\|\bm{\theta}_{u,t}
-\bm{\theta}_{u,*}\|_2$,
where  $\bm{\theta}_{u,t}$ and $\bm{\theta}_{u,*}$ represent the learned and the
ground-truth parameter vectors of user $u$ respectively.
A smaller $\frac{1}{N_u} \sum_u\|\bm{\theta}_{u,t} -\bm{\theta}_{u,*}\|_2$ implies
a higher accuracy in learning of parameter vector.
Figure~\ref{fig:theta_diff} shows the average difference
$\frac{1}{N_u} \sum_u\|\bm{\theta}_{u,t} -\bm{\theta}_{u,*}\|_2$ in \textit{every} 50 iterations.
One can observe that for all algorithms,
the values of $\frac{1}{N_u} \sum_u\|\bm{\theta}_{u,t} -\bm{\theta}_{u,*}\|_2$ decrease in $t$.  
That is, all algorithms can learn the parameters more accurately with more interactions with  users.
Moreover, ConUCB is the best algorithm in learning of the parameters.  

\noindent
{\bf Impact of conversation frequency $b(t)$.} 
Next we study the impact of conversation frequency.
In principle, key-term-level conversations should be less frequent than arm-level interactions, i.e., $b(t) \leq t$.
Thus we mainly consider two types of conversation frequency function $b(t)$:
(1) $b(t)=Q_l\lfloor \log(t) \rfloor$: ask $Q_l$ questions every time,
  while the span between two consecutive conversations gets larger and larger;
(2) $b(t)=Q_l \lfloor \frac{t}{Q_q} \rfloor$: ask $Q_l$ questions
per $Q_q$ iterations.
For the first type of $b(t)$, we vary the value of $Q_l$, 
and obtain cumulative regrets at round 1000 shown in Figure~\ref{fig:bt_1}.
For the second type of $b(t)$, we set $Q_q=50$, vary the value
of $Q_l$, and plot cumulative regrets at round 1000 in Figure~\ref{fig:bt_2}.
We also run the experiments with $Q_q=20$ and $Q_q=100$, and
the results are similar to that with $Q_q=50$.

Figure~\ref{fig:bt_2} and Figure~\ref{fig:bt_1} show that
asking more questions can help 
reduce the cumulative regrets more.  
For example, in Figure~\ref{fig:bt_1},
the cumulative regret is the largest when $b(t)=\lfloor \log(t) \rfloor$,
while the cumulative regret is the smallest when $b(t)=10 \lfloor \log(t) \rfloor$.
Similarly, in Figure~\ref{fig:bt_2}, the cumulative regret is the largest 
when $b(t)= \lfloor \frac{t}{50} \rfloor$, 
while it is the smallest when $b(t)= 10\lfloor \frac{t}{50} \rfloor$.

Comparing the cumulative regrets with $b(t)= 5\lfloor \log(t) \rfloor$ and $b(t)=5 \lfloor \frac{t}{50} \rfloor$, we can observe that although the agent asks more questions 
with $b(t)=5 \lfloor \frac{t}{50} \rfloor$,
its cumulative regret is much larger than that with $b(t)= 5\lfloor \log(t) \rfloor$.
The reason seems to be that with $b(t)= 5\lfloor \log(t) \rfloor$ the agent can ask more questions at the beginning, quickly  capture users' preference, and then gradually reduce the cumulative regret afterward.

\noindent
{\bf Impact of poolsize $|\mathcal{A}_t|$}.
We change the size of $\mathcal{A}_t$ from 25 to 500 while fixing all the parameters as described in Section~\ref{sub_sec:syn_setting}.
Figure~\ref{fig:poolsize} plots the cumulative regrets under different poolsizes.
One can observe that as the poolsize increases,
the cumulative regrets of all algorithms also increase, since it is more difficult  
to select the best arm when the poolsize becomes larger.
Again, we observe similar trends in different poolsizes:
(1) Using conversational feedbacks, 
either by querying additional arms (Arm-Con), 
or by querying key-terms (Var-RS, Var-MRC, Var-LCR, ConUCB) can reduce regret;
Querying key-terms is more effective and 
has smaller regrets than querying arms.
(2) The regret of ConUCB is the smallest.

\begin{figure}[t]
  \vspace{-5mm}
 \centering
 \subfloat[Effect of $b(t)=Q_l\lfloor \log(t) \rfloor$]
 { \includegraphics[width=.48\linewidth, height=.3\linewidth]{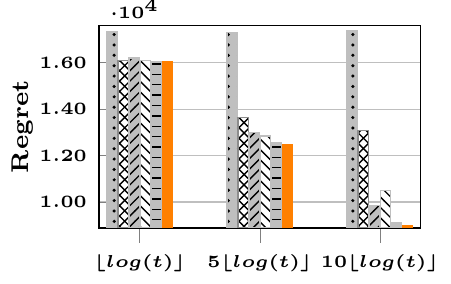}
   \label{fig:bt_1}} \hspace{-4mm}
 \subfloat[Effect of $b(t)=Q_l \lfloor \frac{t}{50} \rfloor$]
 { \includegraphics[width=.48\linewidth, height=.3\linewidth]{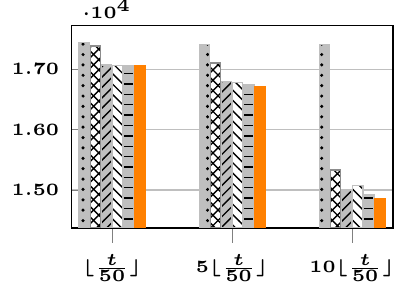}
   \label{fig:bt_2}}\\
 \vspace{-3.5mm}
\subfloat[Effect of poolsize]
 { \includegraphics[width=0.85\linewidth, height=.45\linewidth]{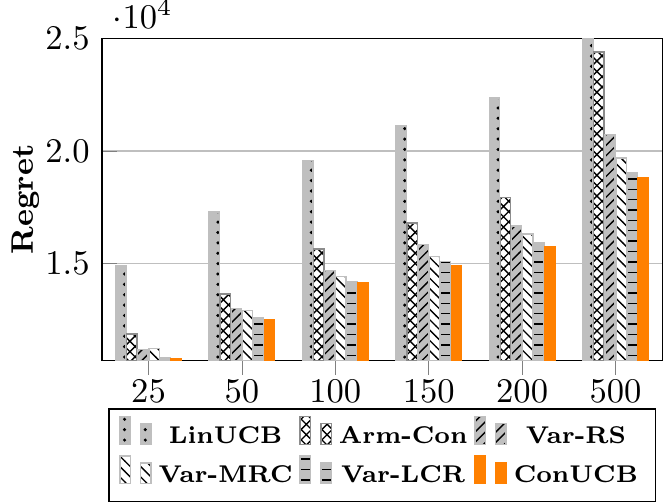}
   \label{fig:poolsize}}
 \vspace{-3.5mm}
\caption{Effect of various factors: Figures (a) \& (b) show
  the effect of $b(t)$; Figure (c) shows the effect of poolsize.}
  \end{figure}

\begin{figure*}[t]
  \centering
  \vspace{-5mm}
  \includegraphics[width=0.8\linewidth]{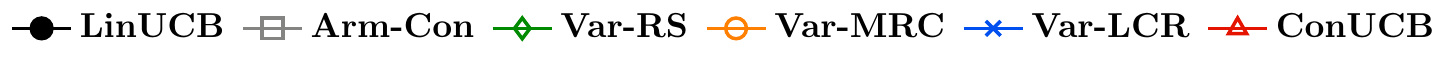}\\
  \vspace{-5mm}
  \subfloat[Cumulative regret on synthetic dataset]
           { \includegraphics[width=.25\linewidth, height=.26\linewidth]{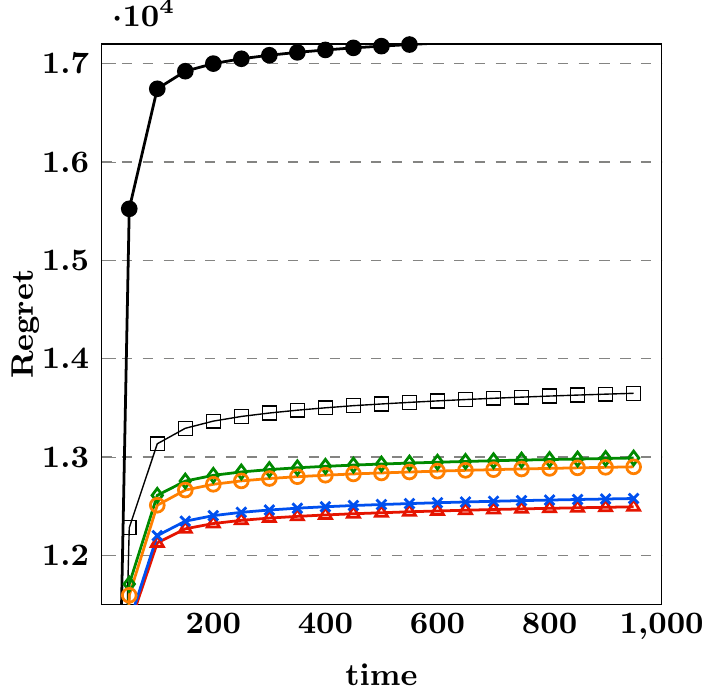}
             \label{fig:syn_regret}}
  \subfloat[Accuracy of learned parameters]
           { \includegraphics[width=.24\linewidth, height=.245\linewidth]{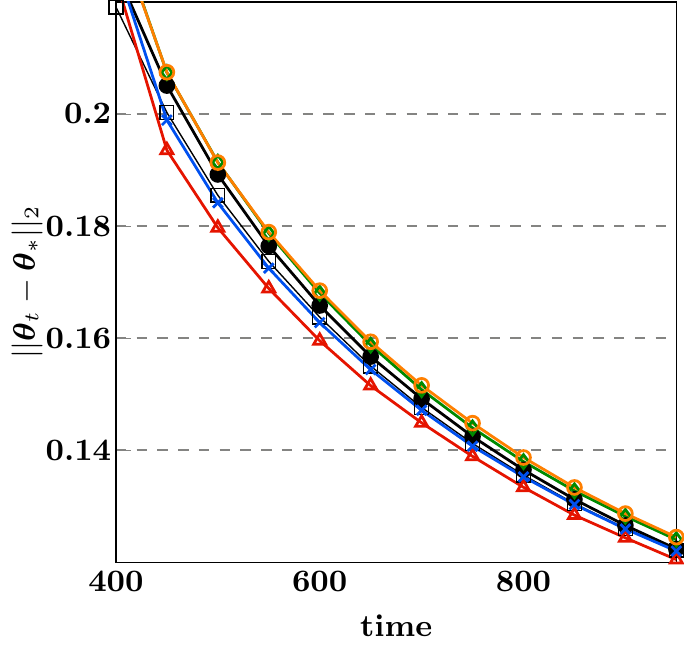}
             \label{fig:theta_diff}}
  \subfloat[Cumulative regret on Yelp dataset]
           { \includegraphics[width=.24\linewidth, height=.26\linewidth]{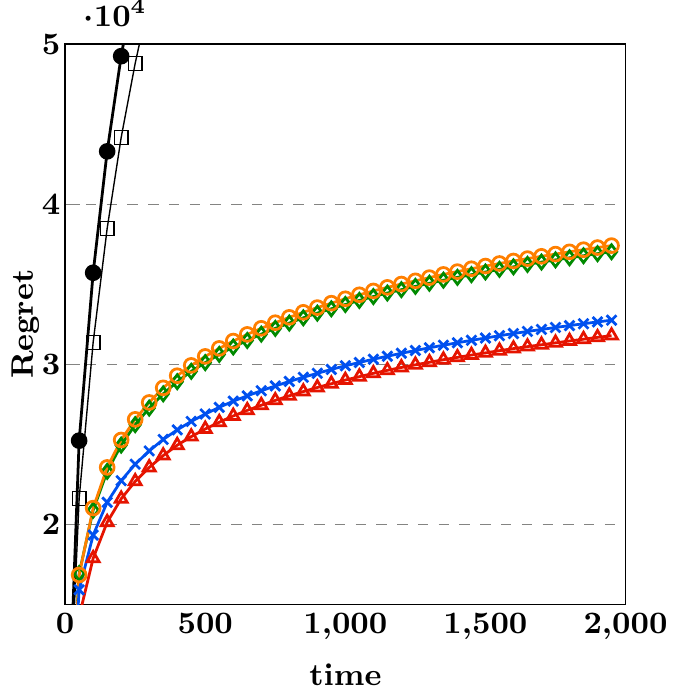}
             \label{fig:yelp_regret}}
  \subfloat[Normalized CTR on Toutiao dataset]
           { \includegraphics[width=.24\linewidth, height=.25\linewidth]{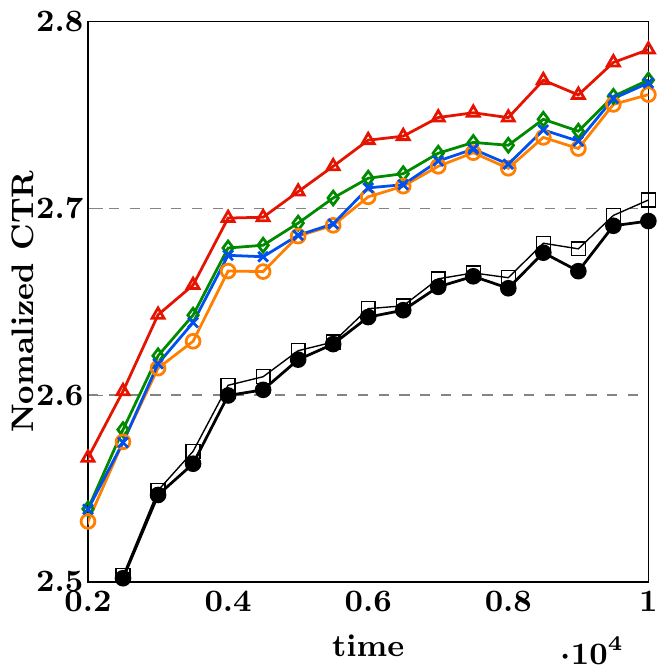}
             \label{fig:toutiao}}
\vspace{-3mm}
    \caption{Experimental results on all datasets.}
\vspace{-6mm}
\end{figure*}

\section{Experiments on Real Data}
\label{sec:real_data}
In this section, we describe empirical evaluations of the ConUCB algorithm
on two real-world datasets.

\subsection{Experiments on Yelp Dataset}
\label{sec:exp_yelp}
\noindent
{\bf Yelp dataset.}
The public Yelp dataset\footnote{http://www.yelp.com/academic\_dataset} contains
users' reviews of restaurants on Yelp.
We only keep users with no less than 100 reviews,
and restaurants with no less than 50 reviews.
The final Yelp dataset has 1,998 users, 25,160 restaurants, and 342,237 reviews.

\noindent
{\bf Experiment settings.} 
We take each restaurant as an arm.
Each restaurant in the dataset is associated with a number of categories. 
For example,  one restaurant named \textit{``Garage''} is associated with the
following categories: \{ \textit{``Mexican'', ``Burgers'', ''Gastropubs''}\}.
We take each category as a key-term.
Each arm is equally related to its associated key-terms.
There are in total 1,004 categories in the dataset, i.e. $|\mathcal{K}|=1004$.
We construct $50$-dimensional arms' contextual vectors via applying PCA on feature
vectors generated from restaurants' attributes, including:
(1) geographic features: 330 different cities;
(2) categorical features: 1,004 different categories;
(3) average rating and total review count;
(4) attributes: 34 different attributes, such as whether the restaurant serves alcohol or WIFI.
We also normalize all the contextual vectors, i.e., $\|\bm{x}_a\|_2=1, \forall a$.
The original 5-scale ratings are converted to a
binary-valued feedback between restaurants and users, 
i.e., high ratings (4 and 5) as positive(1) and low ratings ( $\leq 3$) as negative(0).
We derive users' true parameters based on ridge regression.  
We fix the size of $\mathcal{A}_t$ to 50.
During the simulation, the true arm-level reward $r_{a,t}$ as well as  key-term
level reward $\tilde{r}_{k,t}$ are generated according to Eq.~(\ref{equ:reward})
and Eq.~(\ref{con_feedback_model}) respectively.
Note that key-terms (i.e., categorical features) 
are part of arms' feature sets,  thus the information about key-terms 
is also available for all algorithms. 

\noindent
{\bf Evaluation results.}
We compare all six algorithms in terms of cumulative regret.
We adopt the frequency function $b(t)=5\lfloor log(t) \rfloor$.  
We run the experiments 10 times, and the average result for each algorithm
is shown in Figure~\ref{fig:yelp_regret}.
We observe similar results as those on the synthetic dataset.
That is, all algorithms that query key-terms, i.e., Var-RS, Var-MRC, Var-LCR, ConUCB,
have smaller cumulative regrets than LinUCB and Arm-Con.
Moreover, the ConUCB algorithm has the smallest cumulative regret, followed by Var-LCR, and then Var-MRC and Var-RS.

\subsection{Experiments on Toutiao Dataset}

\noindent
{\bf Toutiao dataset.}
This is a real-world news recommendation dataset, obtained from Toutiao\footnote{https://www.Toutiao.com/}, which is the largest news recommendation platform in China.
The dataset contains 2,000 users' interaction records in December 2017. There are
1,746,335 news articles and 8,462,597 interaction records.

\noindent
{\bf Experiment settings.}
We take each news article as an arm.
Each article is associated with several categories, such as,
\textit{``news\_car''} and \textit{ ``news\_sports''}.  
Each article is also associated with several keywords 
automatically extracted from the article.
We filter the keywords occurring less than 1,000 times. In total, there are $2,384$ keywords and 573 categories.  
We take them as key-terms $\mathcal{K}$.
We assume that an article is equally related to its associated key-terms.
To get the contextual vector of an arm, 
we first represent the article by a vector of 3,469 features (e.g., topic
distribution, categories, keywords, etc.).
We then use PCA to conduct dimension reduction,
and take the first 100
principal components as the contextual vectors, i.e., $d=100$.
We infer each user's feedback on an article through the user's reading behavior:
if the user reads the article, then the feedback is 1,
otherwise the feedback is 0.
The feedback is also the reward of the article.
We infer each user's feedback on key-terms by simulation.
Specifically, we pre-process the interacted articles
of the 2,000 users in November 2017 as above,
and employ ridge regression to infer users' true preference based on interaction records in the period.
Then we generate the ground-truth key-term-level feedbacks according to Eq.~(\ref{con_feedback_model}).
  Note that the information  about key-terms 
  is also available for LinUCB and Arm-Con algorithms,
  in the sense that key-terms (i.e.,
  categories, keywords) are in arms' feature sets. 

\noindent
{\bf Comparison method.}
The unbiased offline evaluation protocol proposed in ~\cite{li2011unbiased} 
is utilized to compare different algorithms.
The unbiased offline evaluation protocol only works 
when the feedback in the system is collected under a random policy.
Hence, we simulate the \textit{random} policy 
of the system by generating a candidate pool as follows.
At each round $t$, we store the article presented to the user ($a_t$)
and its received feedback $r_{a_t}$.  
Then we create $\mathcal{A}_t$ by including 
the served article along with 49 extra articles 
the user has interacted with (hence $|\mathcal{A}_t|=50, \forall t$).  
The 49 extra articles are drawn uniformly at random so that
for any article $a$ the user interacted with, 
if $a$ occurs in some set $\mathcal{A}_t$,
this article will be the one served by the system 1/50 of the times.  
The performance of algorithms is evaluated by 
Click Through-Rate (CTR),  the ratio 
between the number of clicks an algorithm receives 
and the number of recommendations it makes. 
Specifically, we use the average CTR in every 500 iterations (not the cumulative CTR) as the evaluation metric. 
Following~\cite{li2010contextual},
we normalize the resulting CTR from different algorithms
by the corresponding logged random strategy's CTR.

\noindent
{\bf Evaluation results. }
Figure~\ref{fig:toutiao} shows the normalized CTRs of different algorithms over
2000 users.
One can observe that
algorithms that querying key-terms can achieve
higher CTRs than LinUCB and Arm-Con.
Again ConUCB 
achieves the best performance.
Moreover, on the Toutiao dataset, 
Var-MRC and Var-LCR perform worse than
Var-RS.  
This is because they tend to select key-terms related to a large
group of arms repeatedly.  
One can also observe that Arm-Con only outperforms LinUCB slightly.  
This is because Toutiao dataset contains more negative feedbacks than
positive feedbacks, 
and some negative feedbacks are caused by that a user has read
something similar, rather than this user does not like the article.  
However, articles with such negative feedbacks may be queried by Arm-Con using
additional questions with larger probability, 
due to articles with similar contents received positive feedbacks 
from the same user.
This brings disturbance to the learning
of the user's parameter vector,
decreasing the efficiency of additional questions.

\section{Extension of Algorithm}
\label{sec:discussion}
ConUCB incorporates conversations into
  LinUCB.
  We demonstrate that 
  the technique is generic and can be applied to other contextual bandit
  algorithms as well.
  Specifically, 
we show how to extend
the hLinUCB algorithm~\cite{wang2016learning} with
conversational feedbacks.

\noindent
{\bf Conversational hLinUCB.}
The hLinUCB algorithm~\cite{wang2016learning} is one of the recent contextual bandit
algorithms, which utilizes a set of $l$-dimensional hidden features ($\bm{v}_a \in \mathbb{R}^l$) that affects the expected reward, 
in addition to the contextual features ($\bm{x}_{a,t} \in \mathbb{R}^d$).
 Formally, 
 \begin{align}
  \textstyle  r_{a,t}=(\bm{x}_{a,t}, \bm{v}_{a})^T \bm{\theta}_u +\eta_t,  
  \label{equ:hlinucb_reward}
  \vspace{-6mm}
 \end{align}
 where $\bm{\theta}_u \in \mathbb{R}^{l+d}$ is the parameter of user $u$, and $\eta_t$ is drawn from a zero-mean Gaussian distribution $\mathcal{N}(0,\gamma^2)$.

We then design the hConUCB algorithm to extend the hLinUCB algorithm to incorporate conversational feedbacks.
At round $t$, hConUCB first infers user $u$'s current key-term-level preference 
$\bm{\tilde{\theta}}_{u,t}$
solely based on her conversational feedbacks,
and then use $\bm{\tilde{\theta}}_{u,t}$ to guide the
learning of user $u$'s arm-level preference $\bm{\hat{\theta}}_{u,t}$,
with the estimated hidden features $\{\bm{\hat{v}}_{a,t}\}_{a \in \mathcal{A}}$ fixed. 
 Then, fixing $\bm{\tilde{\theta}}_{u,t}$ and $\bm{\hat{\theta}}_{u,t}$, 
 we use both conversational feedbacks and arm-level feedbacks
 to update the hidden features of arms $\{\bm{\hat{v}}_{a,t}\}_{a \in \mathcal{A}}$.  
One can also derive the confidence bound following a similar procedure as in
Lemma~\ref{lem:conf_int},  and choose the arm with the maximal upper confidence bound
value.
The key-term-selection strategy of hConUCB follows the 
procedure in
Theorem~\ref{thm:key-term-selection}. 
Moreover, we modify the baselines in Section~\ref{sub_sec:syn_setting} in a similar way.
For example, similar to Arm-Con, the hArm-Con algorithm conducts conversations by querying
 whether a user likes the item selected by hLinUCB.

\noindent
{\bf Experiment evaluations. }
 We then compare hConUCB with the modified baselines on the previous three datasets.
 The experimental settings on the three datasets are the same as those in  Section~\ref{sec:exp_syn}  and  in Section~\ref{sec:real_data}.
 The only difference is that  in this section,
 features are randomly partitioned into an observable part and a hidden part.
 We fix the dimensionaliy $l$ of hidden features to $5$, and set
 the dimensionality $d$ of observable features to $45$ on synthetic dataset and Yelp dataset,
and $d=95$ on Toutiao dataset.
 We set $\gamma=0.1$.
 Note that only observable features are showed to the algorithms. 
The arm-level feedback $r_{a,t}$ and the key-term-level feedback $\tilde{r}_{k,t}$ are generated similarly as in Section~\ref{sec:exp_syn}  and  in Section~\ref{sec:real_data}, except the arm-level reward model is in  Eq.~(\ref{equ:hlinucb_reward}).

 Figure \ref{fig:h2ucb_syn_regret} and \ref{fig:h2ucb_yelp_regret} show the cumulative regrets
 (defined in Eq.~(\ref{equ:regret})) on synthetic dataset and Yelp dataset
 respectively.
 One can observe that on both datasets,
 the algorithms that query key-terms except hVar-RS 
 outperform LinUCB, hLinUCB and hArm-Con.
 Moreover, hConUCB 
has the smallest cumulative regret.
The poor performance of hVar-RS
is because randomly selecting key-terms 
cannot effectively contribute to the inference of arms' hidden features.
 
 For the experiments on Toutiao dataset, we normalize the CTRs from
 different algorithms by the corresponding CTR of LinUCB.
 Figure~\ref{fig:h2ucb_toutiao} shows the normalized CTRs of different algorithms
 on Toutiao dataset.
 The values on the vertical axis in Figure~\ref{fig:h2ucb_toutiao} are all larger than $1$,
 indicating that all the algorithms outperform LinUCB.
 Also, we can observe that all algorithms that query key-terms have higher CTRs than hLinUCB
 and hArm-Con,
 and hConUCB achieves the highest CTR.
 Moreover, in this experiment, the performance of hArm-Con is
 worse than that of hLinUCB.
 The reason might be the same as that Arm-Con does not outperform
 LinUCB, as shown in Figure ~\ref{fig:toutiao}.
 In summary, the results demonstrate the effectiveness of  conversational mechanism incorporated into bandit algorithms.

\begin{figure}
   \vspace{-6mm}
   \centering
   \subfloat[synthetic dataset]
           { \includegraphics[width=.48\linewidth, height=.5\linewidth]{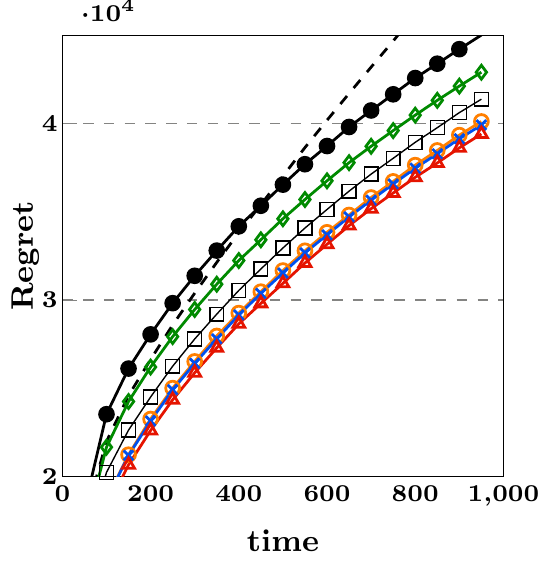}
             \label{fig:h2ucb_syn_regret}}
  \subfloat[Yelp dataset]
           { \includegraphics[width=.48\linewidth, height=.5\linewidth]{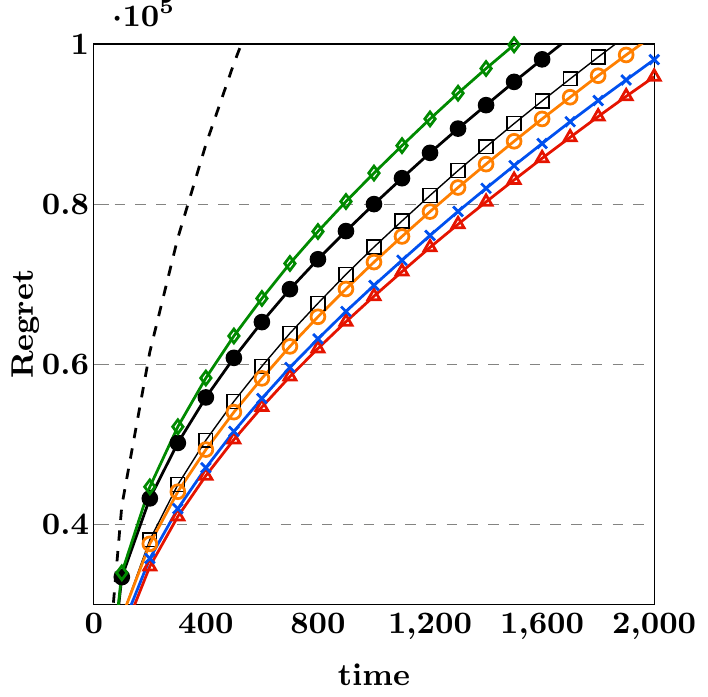}
             \label{fig:h2ucb_yelp_regret}}\\
    \vspace{-4mm}
   \subfloat[Toutiao dataset]
   {
       \includegraphics[width=0.6\linewidth]{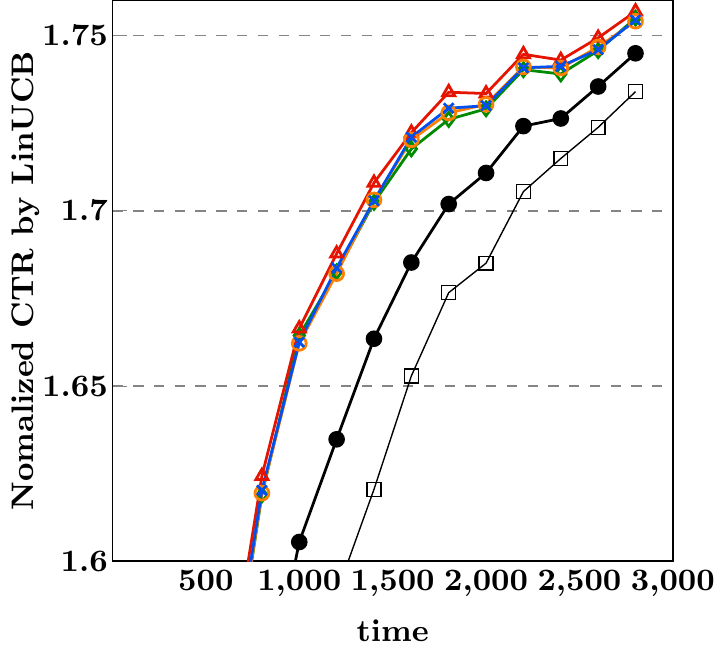}
             \label{fig:h2ucb_toutiao}} \\
   \includegraphics[width=0.8\linewidth]{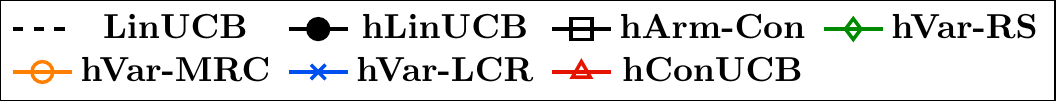}
   \vspace{-4mm}
   \caption{Experimental results of conversational hLinUCB. }
   \label{fig:h2ucb_exp}
 \end{figure}


\section{Related Work}
To the best of our knowledge, this is the first work to study
the conversational contextual bandit problem.
Our work is closely related to the following two lines of research. 

\noindent  
\textbf{Contextual bandit algorithms.}
Contextual bandit is a popular technique to
address the exploration-exploitation trade-off, in various application tasks such as recommendation. 
LinUCB~\cite{li2010contextual} and Thompson Sampling~\cite{agrawal2013thompson}
are two representative algorithms for contextual bandits.
A large number of algorithms have been proposed to accelerate the learning speed 
of LinUCB.  
For example, Bianchi et. al.~\cite{cesa2013gang,wu2016contextual,li2016collaborative}
leverage relationship among users.
Wang et. al.~\cite{wang2016learning} learn hidden features,
and Zheng et.al.~\cite{zeng2016online} make use of a time-changing reward function.  
The CoFineUCB~\cite{yue2012hierarchical} algorithm proposed by Yue et. al. performs a coarse-to-fine hierarchical
exploration.
Different from them, our work tries to accelerate the learning speed through
conversation.
More importantly, 
as shown in Section~\ref{sec:discussion},
our approach is generic and can be applied to many existing algorithms.

It should be possible to incorporate the conversation mechanism into
Thompson Sampling~\cite{agrawal2013thompson}.
Thompson Sampling is a probability matching algorithm that samples $\bm{\theta}_t$ from
the posterior distribution.
Thus, one can define a hierarchical sampling approach that first samples
$\bm{\tilde{\theta}}_t$ according to conversational feedbacks,
and then samples $\bm{\theta}_t$ around $\bm{\tilde{\theta}}_t$ while considering
arm-level feedbacks.

\noindent
\textbf{Conversational recommender systems.} 
Christakopoulou \textit{et. al.}~\cite{christakopoulou2016towards} introduce the idea of conversational recommender systems. They conduct conversations by querying
whether the user likes the items selected by the bandits algorithm.
As shown in this paper, item-level conversations are less efficient
than key-term-level conversations.  
Other researchers~\cite{zhang2018towards,sun2018conversational,christakopoulou2018q}
further leverage recent advances in natural
language understanding to generate conversations and assist recommendation.  
Their works utilize conversation in user preference learning 
(via deep learning
or reinforcement learning)
without theoretical guarantee, 
while our work utilizes conversation to speed up contextual bandit learning 
with theoretical guarantee.   
Yu \textit{et. al.} \cite{yu2019visual} propose a dialog approach to speed up 
bandit learning in recommender systems.  
Their dialog approach recommends multiple items to the user and 
the user needs to provide feedback on why she 
likes or dislikes the recommended items.
In this paper, we explore another approach of using conversations, i.e, querying
key-terms occasionally to actively explore the user's preference.
Bu et. al~\cite{bu2018active} also
leverage conversations to assist item recommendation, however, their algorithm is
only applied to offline learning, while our algorithm, based on the bandit technique, is an online learning algorithm. Furthermore, our algorithm adaptively optimizes the question strategy through interactions and conversations.

\vspace{-2mm}
\section{Conclusion}
We formulate the \textit{conversational contextual bandit} problem by
incorporating a conversational mechanism into contextual bandit.
We design the ConUCB algorithm to adaptively optimize the arm selection
strategy and the key-term selection strategy through conversations and
arm-level interactions with the user.
Theoretical analysis shows that ConUCB can achieve a lower regret upper bound. Extensive experiments on synthetic dataset,
real datasets from Yelp and Toutiao demonstrate that ConUCB
indeed has a faster learning speed.
The generality of the approach of incorporating conversations into bandit algorithms 
is also demonstrated.

\section{Acknowledgments}
The work is supported by National Nature Science Foundation of
China (61902042) and the GRF 14201819.

\vspace{-2mm}
\section{Appendix}
\noindent
{\bf Proof of Lemma~\ref{lem:conf_int}:}
\begin{proof}
According to the closed-form solution of $\boldsymbol{\theta}_t$, we can get: 

\resizebox{0.99\columnwidth}{!}{
    \begin{minipage}{\columnwidth}
      \begin{align}
     &   \textstyle
    \boldsymbol{\theta_t} = \boldsymbol{M}_t^{-1} \boldsymbol{b}_t \nonumber \\
    & \textstyle= \left(\lambda
      \sum_{\tau=1}^{t-1} \boldsymbol{x}_{a,\tau}\boldsymbol{x}_{a,\tau}^T+(1-\lambda)\boldsymbol{I} \right)^{-1}( \lambda \sum_{\tau=1}^{t-1}\boldsymbol{x}_{a,\tau}r_{a,\tau}+(1-\lambda)\boldsymbol{\tilde{\theta}}_t) \nonumber \\
    & \textstyle= \boldsymbol{\theta_*} -(1-\lambda)\boldsymbol{M}_t^{-1}\boldsymbol{\theta_*}+\lambda \boldsymbol{M}_t^{-1}(\sum_{\tau=1}^{t-1}\boldsymbol{x}_{a,\tau}\epsilon_{\tau})+(1-\lambda)\boldsymbol{M}_t^{-1}\boldsymbol{\tilde{\theta}}_t \nonumber. 
  \end{align}
\end{minipage}}

The equality holds since $r_{a,\tau}=\bm{x}_{a,\tau}^T\bm{\theta}_*+\epsilon_{\tau}$.
  With $\boldsymbol{\theta}_* \approx \boldsymbol{\tilde{\theta}}_*$,

  \resizebox{0.95\columnwidth}{!}{
    \begin{minipage}{\columnwidth}
\begin{align}
  \textstyle
  \quad \boldsymbol{\theta}_t -\boldsymbol{\theta}_*&
  \textstyle =(1-\lambda)\boldsymbol{M}_t^{-1}(\boldsymbol{\tilde{\theta}}_t-\boldsymbol{\tilde{\theta}}_*) + \lambda \boldsymbol{M}_t^{-1}(\sum_{\tau=1}^{t-1}\boldsymbol{x}_{a,\tau}\epsilon_{\tau}). 
  \label{equ:theta_diff}
\end{align}
\end{minipage}}

Thus, with  $\|\boldsymbol{\tilde{\theta}}_t-\boldsymbol{\tilde{\theta}}_*\|_{\boldsymbol{\tilde{M}}_t}
\leq \tilde{\alpha}_t$,  we can get:

\resizebox{0.99\columnwidth}{!}{
    \begin{minipage}{\columnwidth}
 \begin{align}
     & \textstyle |\boldsymbol{x}_{a,t}^T\boldsymbol{\theta_t}-\boldsymbol{x}_{a,t}^T\boldsymbol{\theta}_*|
\nonumber \\
& = (1-\lambda)|\boldsymbol{x}_{a,t}^T\boldsymbol{M}_t^{-1}(\boldsymbol{\tilde{\theta}}_t-\boldsymbol{\tilde{\theta}}_*)| + \lambda|\boldsymbol{x}_{a,t}^T\boldsymbol{M}_t^{-1}(\sum_{\tau=1}^{t-1}\boldsymbol{x}_{a,\tau}\epsilon_{\tau})| \nonumber \\
     &\textstyle \leq (1-\lambda)\tilde{\alpha}_t  \|\boldsymbol{x}_{a,t}^T\boldsymbol{M}_t^{-1}\|_{\boldsymbol{\tilde{M}}_t^{-1}} +\lambda \|\boldsymbol{x}_{a,t}\|_{\boldsymbol{M}_t^{-1}}\|\sum_{\tau=1}^{t-1}\boldsymbol{x}_{a,\tau}\epsilon_{\tau}\|_{\boldsymbol{M}_t^{-1}} \nonumber.
 \end{align}
\end{minipage}}

    Let $\alpha_t$ denote the upper bound of
    $\|\sum_{\tau=1}^{t-1}\boldsymbol{x}_{a,\tau}\epsilon_{\tau}\|_{\boldsymbol{M}_t^{-1}}$,
    then Theorem 1 in paper~\cite{abbasi2011improved} suggests that

 \resizebox{0.99\columnwidth}{!}{
    \begin{minipage}{\columnwidth}
 \begin{align}
\textstyle \alpha_t = \sqrt{2\log\left( \frac{\det(\bm{M}_t)^{1/2} \det((1-\lambda)\bm{I})^{-1/2}}{\sigma} \right)} \leq \sqrt{d\log\left( (1+\frac{\lambda t}{(1-\lambda)d})/\sigma \right)}. \nonumber
 \end{align}
\end{minipage}}\\

This proof is then complete.
\end{proof}

\noindent
{\bf Proof of Theorem~\ref{thm:key-term-selection}:}
\begin{proof}
  According to Theorem~\ref{thm:optimal_ub}, to minimize  $\mathbb{E} [ \|\boldsymbol{X}_t\boldsymbol{\theta}_t
 -\boldsymbol{X}_t\boldsymbol{\theta}_*\|_2^2 ]$, the key-term
selection strategy needs to select key-term $k$ 
to minimize 
\begin{equation}
 \hspace{-5mm} \tr\left(
     \boldsymbol{X}_t\boldsymbol{M}_t^{-1}(\boldsymbol{\tilde{M}}_{t-1}+\boldsymbol{\tilde{x}}_{k,t}\boldsymbol{\tilde{x}}_{k,t}^T)^{-1}\boldsymbol{M}_t^{-1}\boldsymbol{X}_t^T \nonumber
    \right). 
\end{equation} 
Then, using Woodbury matrix identity, it is equal to minimize: 
\begin{align}
     \textstyle  \tr\left(\boldsymbol{X}_t \boldsymbol{M}_t^{-1}\boldsymbol{\tilde{M}}_{t-1}^{-1}\boldsymbol{M}_t^{-1}\boldsymbol{X}_t^T \right) -\frac{\| \boldsymbol{X}_t\boldsymbol{M}_t^{-1}\boldsymbol{\tilde{M}}_{t-1}^{-1}\boldsymbol{\tilde{x}}_{k,t}\|_2^2}{ 1+\boldsymbol{\tilde{x}}_{k,t}^T\boldsymbol{\tilde{M}}_{t-1}^{-1}\boldsymbol{\tilde{x}}_{k,t}}. \nonumber
\end{align}    
With the interaction history, $\tr\left(\boldsymbol{X}_t
  \boldsymbol{M}_t^{-1}\boldsymbol{\tilde{M}}_{t-1}^{-1}\boldsymbol{M}_t^{-1}\boldsymbol{X}_t^T
\right)$ is a constant.
Thus, the proof is complete.
 \end{proof}

\begin{thm}
  Given the interaction history at both arm-level and key-term level up to round $t$, 
  we have:
  \begin{align}
  \textstyle & (1) \quad  \min_k   \mathbb{E} [ \|\boldsymbol{X}_t\boldsymbol{\theta}_t -\boldsymbol{X}_t\boldsymbol{\theta}_*\|_2^2 ]
    \Leftrightarrow 
               \min_k  \mathbb{E} [ \|\boldsymbol{X}_t \boldsymbol{M}_t^{-1} (\boldsymbol{\tilde{\theta}}_t -\boldsymbol{\tilde{\theta}}_*)\|_2^2 ]. \nonumber \\
   \textstyle & (2) \quad \mathbb{E} [ \|\boldsymbol{X}_t \boldsymbol{M}_t^{-1} (\boldsymbol{\tilde{\theta}}_t -\boldsymbol{\tilde{\theta}}_*)\|_2^2] 
     \leq  ( \|\boldsymbol{\theta}_*\|_2^2 +1) \nonumber  \tr\left(\boldsymbol{X}_t\boldsymbol{M}_t^{-1}\boldsymbol{\tilde{M}}_{t}^{-1}\boldsymbol{M}_t^{-1}\boldsymbol{X}_t^T
    \right). \nonumber 
   \end{align}
   \label{thm:optimal_ub}
\end{thm}
\begin{proof}
      With Eq.~(\ref{equ:theta_diff}), we can get:
      \begin{align}
       & \quad \textstyle \min_k \mathbb{E}\|\boldsymbol{X}_t\boldsymbol{\theta}_t -\boldsymbol{X}_t\boldsymbol{\theta}_*\|_2^2 \nonumber \\
      & \textstyle =  \min_k \mathbb{E}\|(1-\lambda)\boldsymbol{X}_t \boldsymbol{M}_t^{-1}(\boldsymbol{\tilde{\theta}}_t- \boldsymbol{\tilde{\theta}}_*) +\lambda \boldsymbol{X}_t \boldsymbol{M}_t^{-1}(\sum_{\tau=1}^{t-1}\boldsymbol{x}_{\tau}\epsilon_{\tau})\|  \nonumber . 
      \end{align}
      Note that key-term selection does not affect $\{\epsilon_{\tau}\}_{\tau=1}^{t-1}$ , thus
      we can get the first observation:
\begin{align}
    \textstyle  \min_k   \mathbb{E}\|\boldsymbol{X}_t\boldsymbol{\theta}_t -\boldsymbol{X}_t\boldsymbol{\theta}_*\|_2^2
    \Leftrightarrow 
    \min_k   \mathbb{E}\|\boldsymbol{X}_t \boldsymbol{M}_t^{-1} (\boldsymbol{\tilde{\theta}}_t -\boldsymbol{\tilde{\theta}}_*)\|_2^2. \nonumber
 \end{align}
  For the second observation, according to the closed-form of $\bm{\tilde{\theta}}_t$, we can further infer that:
    \begin{align}
      & \textstyle \boldsymbol{\tilde{\theta}}_t -\boldsymbol{\tilde{\theta}}_*  =-\tilde{\lambda}\boldsymbol{\tilde{M}}_t^{-1}\boldsymbol{\tilde{\theta}}_*+\boldsymbol{\tilde{M}}_t^{-1} \left(\sum_{\tau=1}^t \sum_{k \in \mathcal{K}_{\tau}}  \boldsymbol{\tilde{x}}_{k,\tau}\tilde{\epsilon}_{k,\tau} \right) . \nonumber 
    \end{align}
    Thus, we can get:
    
  \resizebox{0.99\columnwidth}{!}{
    \begin{minipage}{\columnwidth}
    \begin{align}
      & \textstyle \mathbb{E}\|\boldsymbol{X}_t \boldsymbol{M}_t^{-1}( \boldsymbol{\tilde{\theta}}_t -\boldsymbol{\tilde{\theta}}_*)\|_2^2 \nonumber\\
      & \textstyle \leq \tilde{\lambda} \underbrace{\|\boldsymbol{X}_t \boldsymbol{M}_t^{-1} \boldsymbol{\tilde{M}}_t^{-1} \boldsymbol{\tilde{\theta}}_*\|_2^2}_{A_1} + \underbrace{\mathbb{E}\|\boldsymbol{X}_t \boldsymbol{M}_t^{-1} \boldsymbol{\tilde{M}}_t^{-1}(\sum_{\tau=1}^t \sum_{k \in \mathcal{K}_{\tau}}  \boldsymbol{\tilde{x}}_{k,\tau}\tilde{\epsilon}_{k,\tau})\|_2^2}_{A_2} \nonumber
    \end{align}
  \end{minipage}}
  \begin{itemize}
  \item \textbf{Bound $A_1$.} The first term is bounded by:
    \begin{equation}
      \begin{aligned}
      \textstyle  \|\boldsymbol{X}_t \boldsymbol{M}_t^{-1} \boldsymbol{\tilde{M}}_t^{-1} \boldsymbol{\tilde{\theta}}_*\|_2^2 & \textstyle \leq \|\boldsymbol{X}_t \boldsymbol{M}_t^{-1} \boldsymbol{\tilde{M}}_t^{-1/2}\|_F^2 \|\boldsymbol{\tilde{M}}_t^{-1/2}\boldsymbol{\tilde{\theta}}_*\|_2^2 \\ \nonumber
        & \textstyle \leq \frac{\|\bm{\tilde{\theta}}_*\|_2^2}{\tilde{\lambda}}
\tr(\boldsymbol{X}_t \boldsymbol{M}_t^{-1}\boldsymbol{\tilde{M}}_t^{-1}\boldsymbol{M}_t^{-1}\boldsymbol{X}_t^T)        \end{aligned}
    \end{equation}
\item \textbf{Bound $A_2$.} We denote $\sum_{\tau=1}^t \sum_{k \in \mathcal{K}_{\tau}} 
  \boldsymbol{\tilde{x}}_{k,\tau}\tilde{\epsilon}_{k,\tau}=\boldsymbol{\tilde{X}}_t\boldsymbol{\tilde{\epsilon}}_t$,
  where $\boldsymbol{\tilde{X}}_t \in R^{ d \times b(t)}$. The second term is bounded by:
  \begin{equation}
    \begin{aligned}
      & \textstyle \mathbb{E}\|\boldsymbol{X}_t \boldsymbol{M}_t^{-1}\boldsymbol{\tilde{M}}_t^{-1}(\sum_{\tau=1}^t \sum_{k \in \mathcal{K}_{\tau}}  \boldsymbol{\tilde{x}}_{k,\tau}\tilde{\epsilon}_{k,\tau})\|_2^2 \\
      & \textstyle=\mathbb{E}[\tr( \boldsymbol{X}_t \boldsymbol{M}_t^{-1} \boldsymbol{\tilde{M}}_t^{-1} \boldsymbol{\tilde{X}}_t \boldsymbol{\tilde{\epsilon}}_t \boldsymbol{\tilde{\epsilon}}_t^T \boldsymbol{\tilde{X}}_t^T  \boldsymbol{\tilde{M}}_t^{-1}\boldsymbol{M}_t^{-1}\boldsymbol{X}_t^T)] \\ \nonumber
      & \textstyle \leq \tr(\boldsymbol{X}_t\boldsymbol{M}_t^{-1}\boldsymbol{\tilde{M}}_t^{-1}\boldsymbol{M}_t^{-1}\boldsymbol{X}_t^T)
    \end{aligned}
  \end{equation}
  The first inequality is due to $\mathbb{E}[ \boldsymbol{\tilde{\epsilon}}_t
  \boldsymbol{\tilde{\epsilon}}_t^T ]\leq I$.
  \end{itemize}
  Finally, with $\bm{\tilde{\theta}}_* \approx \bm{\theta}_*$, we can finish the proof.
\end{proof}

{\bf Proof of Lemma~\ref{lem:fixed_theta_bound}:}
 \begin{proof}
  In the following analysis, we take
  $\boldsymbol{\tilde{x}}_{k,\tau}=\frac{\sum_{a \in \mathcal{A}}
    w_{a,k}\boldsymbol{x}_{a,\tau}}{\sum_{a \in \mathcal{A}} w_{a,k}}$.

  \resizebox{\columnwidth}{!}{
    \begin{minipage}{\columnwidth}
    \begin{align}
     \textstyle  \boldsymbol{\tilde{\theta}}_t & \textstyle = \boldsymbol{\tilde{M}_t}^{-1} \boldsymbol{b}_t  \nonumber \\
    & \textstyle  =\boldsymbol{\tilde{\theta}}_*-\tilde{\lambda}\boldsymbol{\tilde{M}}_t^{-1}\boldsymbol{\tilde{\theta}}_* + \boldsymbol{\tilde{M}}_t^{-1}\left(\sum_{\tau=1}^t \sum_{k \in \mathcal{K}_{\tau}}  \boldsymbol{\tilde{x}}_{k,\tau}\tilde{\epsilon}_{k,\tau} \right) \nonumber
    \end{align}
  \end{minipage}}
Thus, for a fixed $x$,

   \resizebox{\columnwidth}{!}{
    \begin{minipage}{\columnwidth}
     \begin{align}
     \textstyle |\boldsymbol{x}^T\boldsymbol{\tilde{\theta}}_t -\boldsymbol{x}^T\boldsymbol{\tilde{\theta}_*}| & \leq \tilde{\lambda}|\boldsymbol{x}^T\boldsymbol{\tilde{M}}_t^{-1}\boldsymbol{\tilde{\theta}}_*|+ |\boldsymbol{x}^T\boldsymbol{\tilde{M}}_t^{-1} \left( \sum_{\tau=1}^t \sum_{k \in \mathcal{K}_{\tau}}  \boldsymbol{\tilde{x}}_{k,\tau}\tilde{\epsilon}_{k,\tau} \right)| \nonumber . 
     \end{align}
   \end{minipage}} 
   Since $|\boldsymbol{x}^T\boldsymbol{\tilde{M}}_t^{-1}\boldsymbol{\tilde{\theta}}_*| \leq \|\boldsymbol{x}^T\boldsymbol{\tilde{M}}_t^{-1}\|_2\|\boldsymbol{\tilde{\theta}}_*\|_2  \leq \frac{\|\boldsymbol{x}\|_{\boldsymbol{\tilde{M}}_t^{-1}}}{\sqrt{\tilde{\lambda}}} \|\boldsymbol{\tilde{\theta}}_*\|_2$, we can bound the first term: $\tilde{\lambda}|\boldsymbol{x}^T\boldsymbol{\tilde{M}}_t^{-1}\boldsymbol{\tilde{\theta}}_*| \leq \sqrt{\tilde{\lambda}} \|\boldsymbol{x}\|_{\boldsymbol{\tilde{M}}_t^{-1}} \|\boldsymbol{\tilde{\theta}}_*\|_2$.
   We next try to  bound the second term.
   According to the key-term-selection strategy in Eq.~(\ref{equ:optimized_strategy}), we can get
  $$ \textstyle \mathbb{E}[ \sum_{\tau=1}^t \sum_{k \in \mathcal{K}_{\tau}} \tilde{\epsilon}_{k,\tau}]= \sum_{\tau=1}^t \sum_{k \in \mathcal{K}_{\tau}} \mathbb{E}[\tilde{\epsilon}_{k,\tau}]=0,$$ 
  and thus by Azuma's inequality, for a fixed $x$ at round $t$, with $\alpha=\sqrt{\frac{1}{2} \log \frac{2}{\sigma}}$, we have
  \begin{align}
    & \textstyle \mathbb{P} \left(|\boldsymbol{x}^T\boldsymbol{\tilde{M}}_t^{-1}(\sum_{\tau=1}^t \sum_{k \in \mathcal{K}_{\tau}} \boldsymbol{\tilde{x}}_{k,\tau}\tilde{\epsilon}_{k,\tau})| \geq \alpha \|\boldsymbol{x}\|_{\boldsymbol{\tilde{M}}_t^{-1}} \right) \nonumber \\
    & \textstyle \leq 2 \exp\left(-\frac{2\alpha^2\boldsymbol{x}^T\boldsymbol{\tilde{M}}_t^{-1}\boldsymbol{x}}{\sum_{\tau=1}^t \sum_{k \in \mathcal{K}_{\tau}}  (\boldsymbol{x}^T\boldsymbol{\tilde{M}}_t^{-1}\boldsymbol{\tilde{x}}_{k,\tau})^2} \right) \leq  2\exp(-2\alpha^2)) =\sigma \nonumber,
  \end{align}
 since

   \resizebox{0.99\columnwidth}{!}{
    \begin{minipage}{\columnwidth}
  \begin{align}
  & \textstyle  \boldsymbol{x}^T\boldsymbol{\tilde{M}}_t^{-1}\boldsymbol{x}  \textstyle = \boldsymbol{x}^T\boldsymbol{\tilde{M}}_t^{-1} \left( \tilde{\lambda}\boldsymbol{I} +\sum_{\tau=1}^t \sum_{k \in \mathcal{K}_{\tau}} \boldsymbol{\tilde{x}}_{k,\tau}\boldsymbol{\tilde{x}}_{k,\tau}^T \right)\boldsymbol{\tilde{M}}_t^{-1}\boldsymbol{x} \nonumber \\ \nonumber
    & \textstyle \leq \sum_{\tau=1}^t \sum_{k \in \mathcal{K}_{\tau}}  (\boldsymbol{x}^T\boldsymbol{\tilde{M}}_t\boldsymbol{\tilde{x}}_{k,\tau})^2.
  \end{align}
  \end{minipage}} 
  Thus,  for a fixed $x$ and fixed $t$, with probability $1-\sigma$,
  \begin{equation}
   \begin{aligned}
     \textstyle  \left<\boldsymbol{x}^T,\boldsymbol{\tilde{\theta}}_t-\boldsymbol{\tilde{\theta}}_* \right>& \leq \left(\sqrt{\tilde{\lambda}}\|\boldsymbol{\tilde{\theta}}_*\|_2+\sqrt{\frac{1}{2} \log \frac{2}{\sigma}} \right) \|\boldsymbol{x}\|_{\boldsymbol{\tilde{M}}_t^{-1}} \nonumber .
     \label{equ:x_theta_bound_f}
   \end{aligned}
 \end{equation}
     Next, using thee above bound,  we can bound
     $\|\bm{\tilde{\theta}}_t -\bm{\tilde{\theta}}_*\|_{\bm{\tilde{M}}_t}$, where:
     \begin{equation}
       \textstyle  \|\boldsymbol{\tilde{\theta}}_t -\boldsymbol{\tilde{\theta}}_*\|_{\boldsymbol{\tilde{M}}_t}=<\boldsymbol{\tilde{M}}_t^{1/2}\boldsymbol{X}, \boldsymbol{\tilde{\theta}}_t-\boldsymbol{\tilde{\theta}}_*>, \boldsymbol{X}=\frac{\boldsymbol{\tilde{M}}_t^{1/2}(\boldsymbol{\tilde{\theta}}_t-\boldsymbol{\tilde{\theta}}_*)}{ \|\boldsymbol{\tilde{\theta}}_t -\boldsymbol{\tilde{\theta}}_*\|_{\boldsymbol{\tilde{M}}_t}} \nonumber.
     \end{equation}

     We follow the \textit{covering argument} in Chapter 20 of \cite{lattimore2018bandit} to prove.
     First, we identify a finite set $C_{\epsilon} \subset R^d$ such that
     whatever value $\boldsymbol{X}$ takes, there exists some $\boldsymbol{x} \in C_{\epsilon}$ that are
     $\epsilon$-close to $\boldsymbol{X}$.
     By definition, we have $\|\boldsymbol{X}\|_2^2=1$, which means $\boldsymbol{X} \in S^{d-1}=\{\boldsymbol{x} \in
     R^d: \|\boldsymbol{x}\|_2=1\}$.
     Thus, it is sufficient to cover $S^{d-1}$.
     Tor-lattimore et. al.~\cite{lattimore2018bandit} has proven the following Lemma.
     \begin{lem}
       There exists a set $C_{\epsilon} \subset R^d$ with $|C_{\epsilon}|\leq
       (3/\epsilon)^d$ such that for all $x \in S^{d-1}$, there exist a $y \in
       C_{\epsilon}$ with $\|x-y\|\leq \epsilon$.
     \end{lem}
     Then we apply a union bound for the elements in $C_{\epsilon}$, we have:
     
       \resizebox{\columnwidth}{!}{
         \begin{minipage}{\columnwidth}
        \begin{align}
          \textstyle \mathbb{P}\left( \exists \boldsymbol{x}\in C_{\epsilon}, \left< \boldsymbol{\tilde{M}}_t^{1/2}x,\boldsymbol{\tilde{\theta}}_t-\boldsymbol{\tilde{\theta}}_* \right> \geq \left(\sqrt{\tilde{\lambda}}\|\boldsymbol{\tilde{\theta}}_*\|_2+\sqrt{\frac{1}{2} \log \frac{2|C_{\epsilon}|}{\sigma}}\right) \right) \leq \sigma \nonumber .
       \end{align}
     \end{minipage}}
   Then
   
     \resizebox{1.1\columnwidth}{!}{
         \begin{minipage}{\columnwidth}
       \begin{align}
         & \textstyle \|\boldsymbol{\tilde{\theta}}_t-\boldsymbol{\tilde{\theta}}_*\|_{\boldsymbol{\tilde{M}}_t} =\max_{\boldsymbol{x} \in S^{d-1}} \left< \boldsymbol{\tilde{M}}_t^{1/2}\boldsymbol{x}, \boldsymbol{\tilde{\theta}}_t-\boldsymbol{\tilde{\theta}}_* \right> \nonumber \\
         & \textstyle =\max_{\boldsymbol{x} \in S^{d-1}}\min_{\boldsymbol{y} \in C_{\epsilon}} \left[  \left< \boldsymbol{\tilde{M}}_t^{1/2}(\boldsymbol{x}-\boldsymbol{y}), \boldsymbol{\tilde{\theta}}_t-\boldsymbol{\tilde{\theta}}_* \right> + \left< \boldsymbol{\tilde{M}}_t^{1/2}\boldsymbol{y}, \boldsymbol{\tilde{\theta}}_t-\boldsymbol{\tilde{\theta}}_* \right>  \right] \nonumber \\
         & \textstyle \leq \max_{\boldsymbol{x} \in S^{d-1}} \min_{\boldsymbol{y} \in C_{\epsilon}} \left[\|\boldsymbol{\tilde{\theta}}_t-\boldsymbol{\tilde{\theta}}_*\|_{\boldsymbol{\tilde{M}}_t}\|\boldsymbol{x}-\boldsymbol{y}\|_2+
          \sqrt{\tilde{\lambda}}\|\boldsymbol{\tilde{\theta}}_*\|_2+\sqrt{\frac{1}{2} \log \frac{2|C_{\epsilon}|}{\sigma}}
           \right]  \nonumber \\
           & \textstyle \leq \epsilon  \|\boldsymbol{\tilde{\theta}}_t-\boldsymbol{\tilde{\theta}}_*\|_{\boldsymbol{\tilde{M}}_t} +
 \sqrt{\tilde{\lambda}}\|\boldsymbol{\tilde{\theta}}_*\|_2+\sqrt{\frac{1}{2} \log \frac{2|C_{\epsilon}|}{\sigma}} \nonumber .
        \end{align}
      \end{minipage}}
  We set $\epsilon=\frac{1}{2}$. Up to round $t$, we only update $\bm{\tilde{\theta}}_t$ at most $b(t)$ times, thus by the union bound, at each round $t$, with probability $1-\sigma$:
   \begin{equation}
\textstyle \quad \|\boldsymbol{\tilde{\theta}}_t-\boldsymbol{\tilde{\theta}}_*\|_{\boldsymbol{\tilde{M}}_t} \leq  \sqrt{2 \left (d \log6+\log (\frac{2b(t)}{\sigma}) \right)} +   2\sqrt{\tilde{\lambda}}\|\boldsymbol{\tilde{\theta}}_*\|_2 \nonumber .
       \end{equation}
This proof is then complete.  
\end{proof}

\noindent
{\bf Proof of Theorem~\ref{thm:regret_bound_fixed_scenario}:}
\vspace{-2mm}
\begin{proof}
Let $a_t^*$ denote the best arm at round $t$, and $c_t(\bm{x}_{a,t})=\alpha_t \|\bm{x}_{a,t}\|_{\bm{M}_t^{-1}}$, and $\tilde{c}_t(\bm{x}_{a,t})=\tilde{\alpha}_t
    \|\boldsymbol{x}_{a,t}^T\boldsymbol{M}_t^{-1}\|_{\boldsymbol{\tilde{M}}_t^{-1}}$.
Then the regret at round $t$ is:

 \resizebox{0.99\columnwidth}{!}{
    \begin{minipage}{\columnwidth}
 \begin{align}
   \textstyle R_t&= \textstyle \boldsymbol{\theta}_*^T\boldsymbol{x}_{a_t^*,t}- \boldsymbol{\theta}_*^T\boldsymbol{x}_{a_t,t}  \nonumber \\ 
    &\textstyle =\boldsymbol{\theta}_*^T\boldsymbol{x}_{a_t^*,t}-\boldsymbol{\theta}_t^T\boldsymbol{x}_{a_t^*,t}+\boldsymbol{\theta}_t^T\boldsymbol{x}_{a_t^*,t}+c_t(\boldsymbol{x}_{a_t^*,t}) + \tilde{c}_t(\boldsymbol{x}_{a_t^*,t}) \nonumber \\
    & \textstyle \quad -c_t(\boldsymbol{x}_{a_t^*,t}) - \tilde{c}_t(\boldsymbol{x}_{a_t^*,t})-\boldsymbol{\theta}_*^T\boldsymbol{x}_{a_t,t} \nonumber \\ 
    & \textstyle \leq (\boldsymbol{\theta}_*-\boldsymbol{\theta}_t)^T\boldsymbol{x}_{a_t^*,t}+ \boldsymbol{\theta}_t^T\boldsymbol{x}_{a_t,t} +c_t(\boldsymbol{x}_{a_t,t}) + \tilde{c}_t(\boldsymbol{x}_{a_t,t}) \nonumber \\ 
    & \textstyle \quad -c_t(\boldsymbol{x}_{a_t^*,t}) - \tilde{c}_t(\boldsymbol{x}_{a_t^*,t})-\boldsymbol{\theta}_*^T\boldsymbol{x}_{a_t,t} \tag{1} \\
    & \textstyle \leq 2 c_t(\boldsymbol{x}_{a_t,t}) + 2\tilde{c}_t(\boldsymbol{x}_{a_t,t}) \tag{2} \nonumber
 \end{align}
\end{minipage}}

\noindent
  Inequality $(1)$ is due to the arm selection strategy, i.e., $$a_t=\arg \max_{a\in \mathcal{A}_t} \boldsymbol{x}_{a,t}^T\boldsymbol{\theta}_t+c_t(\boldsymbol{x}_{a,t})+\tilde{c}_t(\boldsymbol{x}_{a,t}).$$
  Inequality $(2)$ is from
  $|\boldsymbol{x}_{a,t}^T(\boldsymbol{\theta}_t -\boldsymbol{\theta}_*)| \leq
  c_t(\boldsymbol{x}_{a,t}) +\tilde{c}_t(\boldsymbol{x}_{a,t}), \forall a \in \mathcal{A}$.
   For the following analysis, we abbreviate $\boldsymbol{x}_{a_t,t}$ as $\boldsymbol{x}_t$.
   Thus, the cumulative regret until $T$ is:

   \resizebox{1.05\columnwidth}{!}{
    \begin{minipage}{\columnwidth}
   \begin{align}
      & \textstyle R(T)= \textstyle \sum_{t=1}^{T} R_t \leq 2 \sum_{t=1}^T \left( c_t(\boldsymbol{x}_{a_t,t}) + \tilde{c}_t(\boldsymbol{x}_{a_t,t})\right) \nonumber \\
       & \textstyle \leq 2 \lambda \alpha_T \sqrt{T\sum_{t=1}^T  \|\boldsymbol{x}_t\|_{\boldsymbol{M}_t^{-1}}^2}  + 2(1-\lambda)\tilde{\alpha}_T \sqrt{T\sum_{t=1}^T \|\boldsymbol{x}_t^T\boldsymbol{M}_t^{-1}\|_{\boldsymbol{\tilde{M}}_t^{-1}}^2}  \nonumber .
     \label{equ:regret}
   \end{align}
  \end{minipage}}

\noindent
The last inequality is from the Cauchy-Schwarz inequality, and $\alpha_t, \tilde{\alpha}_t$
are non-decreasing.
Moreover,

  \resizebox{\columnwidth}{!}{
    \begin{minipage}{\columnwidth}
    \begin{equation}
  \begin{aligned}
   \textstyle \|\boldsymbol{x}_t^T\boldsymbol{M}_t^{-1}\|_{\boldsymbol{\tilde{M}}_t^{-1}}^2 =\boldsymbol{x}_t^T\boldsymbol{M}_t^{-1}\boldsymbol{\tilde{M}}_t^{-1}\boldsymbol{M}_t^{-1}\boldsymbol{x}_t
    \leq  \frac{1}{\tilde{\lambda}(1-\lambda)} \|\boldsymbol{x}_t\|_{\boldsymbol{M}_t^{-1}}^2, \label{appendix_equ:x_M_M_M_x}
  \end{aligned}
 \end{equation}
\end{minipage}}
Thus, the cumulative regret is bounded by:

   \resizebox{\columnwidth}{!}{
    \begin{minipage}{\columnwidth}  
\begin{equation}
   \textstyle R(T)  \leq 2 \left( \lambda\alpha_T + \sqrt{\frac{1-\lambda}{\tilde{\lambda}}}\tilde{\alpha}_T \right) \sqrt{T \sum_{t=1}^T \|\boldsymbol{x}_t\|_{\boldsymbol{M}_t^{-1}}^2}.
  \label{equ:cumulative-regret}
\end{equation}
\end{minipage}}

Together with $\alpha_t$ in Lemma~\ref{lem:conf_int}, and 
$\tilde{\alpha}_t$ in
Lemma~\ref{lem:fixed_theta_bound}, as well as Lemma 11 in \cite{abbasi2011improved},  we can finish the proof.
\end{proof}

\bibliographystyle{ACM-Reference-Format}
\bibliography{reference}


\begin{thebibliography}{18}


\ifx \showCODEN    \undefined \def \showCODEN     #1{\unskip}     \fi
\ifx \showDOI      \undefined \def \showDOI       #1{#1}\fi
\ifx \showISBNx    \undefined \def \showISBNx     #1{\unskip}     \fi
\ifx \showISBNxiii \undefined \def \showISBNxiii  #1{\unskip}     \fi
\ifx \showISSN     \undefined \def \showISSN      #1{\unskip}     \fi
\ifx \showLCCN     \undefined \def \showLCCN      #1{\unskip}     \fi
\ifx \shownote     \undefined \def \shownote      #1{#1}          \fi
\ifx \showarticletitle \undefined \def \showarticletitle #1{#1}   \fi
\ifx \showURL      \undefined \def \showURL       {\relax}        \fi
\providecommand\bibfield[2]{#2}
\providecommand\bibinfo[2]{#2}
\providecommand\natexlab[1]{#1}
\providecommand\showeprint[2][]{arXiv:#2}

\bibitem[\protect\citeauthoryear{Abbasi-Yadkori, P{\'a}l, and
  Szepesv{\'a}ri}{Abbasi-Yadkori et~al\mbox{.}}{2011}]%
        {abbasi2011improved}
\bibfield{author}{\bibinfo{person}{Yasin Abbasi-Yadkori},
  \bibinfo{person}{D{\'a}vid P{\'a}l}, {and} \bibinfo{person}{Csaba
  Szepesv{\'a}ri}.} \bibinfo{year}{2011}\natexlab{}.
\newblock \showarticletitle{Improved algorithms for linear stochastic bandits}.
  In \bibinfo{booktitle}{\emph{Advances in Neural Information Processing
  Systems}}. \bibinfo{pages}{2312--2320}.
\newblock


\bibitem[\protect\citeauthoryear{Agrawal and Goyal}{Agrawal and Goyal}{2013}]%
        {agrawal2013thompson}
\bibfield{author}{\bibinfo{person}{Shipra Agrawal} {and} \bibinfo{person}{Navin
  Goyal}.} \bibinfo{year}{2013}\natexlab{}.
\newblock \showarticletitle{Thompson sampling for contextual bandits with
  linear payoffs}. In \bibinfo{booktitle}{\emph{International Conference on
  Machine Learning}}. \bibinfo{pages}{127--135}.
\newblock


\bibitem[\protect\citeauthoryear{Bu and Small}{Bu and Small}{2018}]%
        {bu2018active}
\bibfield{author}{\bibinfo{person}{Yuheng Bu} {and} \bibinfo{person}{Kevin
  Small}.} \bibinfo{year}{2018}\natexlab{}.
\newblock \showarticletitle{Active Learning in Recommendation Systems with
  Multi-level User Preferences}.
\newblock \bibinfo{journal}{\emph{arXiv preprint arXiv:1811.12591}}
  (\bibinfo{year}{2018}).
\newblock


\bibitem[\protect\citeauthoryear{Cesa-Bianchi, Gentile, and
  Zappella}{Cesa-Bianchi et~al\mbox{.}}{2013}]%
        {cesa2013gang}
\bibfield{author}{\bibinfo{person}{Nicolo Cesa-Bianchi},
  \bibinfo{person}{Claudio Gentile}, {and} \bibinfo{person}{Giovanni
  Zappella}.} \bibinfo{year}{2013}\natexlab{}.
\newblock \showarticletitle{A gang of bandits}. In
  \bibinfo{booktitle}{\emph{Advances in Neural Information Processing
  Systems}}. \bibinfo{pages}{737--745}.
\newblock


\bibitem[\protect\citeauthoryear{Christakopoulou, Beutel, Li, Jain, and
  Chi}{Christakopoulou et~al\mbox{.}}{2018}]%
        {christakopoulou2018q}
\bibfield{author}{\bibinfo{person}{Konstantina Christakopoulou},
  \bibinfo{person}{Alex Beutel}, \bibinfo{person}{Rui Li},
  \bibinfo{person}{Sagar Jain}, {and} \bibinfo{person}{Ed~H Chi}.}
  \bibinfo{year}{2018}\natexlab{}.
\newblock \showarticletitle{Q\&R: A two-stage approach toward interactive
  recommendation}. In \bibinfo{booktitle}{\emph{Proceedings of the 24th ACM
  SIGKDD International Conference on Knowledge Discovery \& Data Mining}}. ACM,
  \bibinfo{pages}{139--148}.
\newblock


\bibitem[\protect\citeauthoryear{Christakopoulou, Radlinski, and
  Hofmann}{Christakopoulou et~al\mbox{.}}{2016}]%
        {christakopoulou2016towards}
\bibfield{author}{\bibinfo{person}{Konstantina Christakopoulou},
  \bibinfo{person}{Filip Radlinski}, {and} \bibinfo{person}{Katja Hofmann}.}
  \bibinfo{year}{2016}\natexlab{}.
\newblock \showarticletitle{Towards conversational recommender systems}. In
  \bibinfo{booktitle}{\emph{KDD}}. ACM, \bibinfo{pages}{815--824}.
\newblock


\bibitem[\protect\citeauthoryear{Filippi, Cappe, Garivier, and
  Szepesv{\'a}ri}{Filippi et~al\mbox{.}}{2010}]%
        {filippi2010parametric}
\bibfield{author}{\bibinfo{person}{Sarah Filippi}, \bibinfo{person}{Olivier
  Cappe}, \bibinfo{person}{Aur{\'e}lien Garivier}, {and} \bibinfo{person}{Csaba
  Szepesv{\'a}ri}.} \bibinfo{year}{2010}\natexlab{}.
\newblock \showarticletitle{Parametric bandits: The generalized linear case}.
  In \bibinfo{booktitle}{\emph{NIPS}}. \bibinfo{pages}{586--594}.
\newblock


\bibitem[\protect\citeauthoryear{Lattimore and Szepesv{\'a}ri}{Lattimore and
  Szepesv{\'a}ri}{[n.d.]}]%
        {lattimore2018bandit}
\bibfield{author}{\bibinfo{person}{Tor Lattimore} {and} \bibinfo{person}{Csaba
  Szepesv{\'a}ri}.} \bibinfo{year}{[n.d.]}\natexlab{}.
\newblock \showarticletitle{Bandit algorithms}.
\newblock  (\bibinfo{year}{[n.\,d.]}).
\newblock


\bibitem[\protect\citeauthoryear{Li, Chu, Langford, and Schapire}{Li
  et~al\mbox{.}}{2010}]%
        {li2010contextual}
\bibfield{author}{\bibinfo{person}{Lihong Li}, \bibinfo{person}{Wei Chu},
  \bibinfo{person}{John Langford}, {and} \bibinfo{person}{Robert~E Schapire}.}
  \bibinfo{year}{2010}\natexlab{}.
\newblock \showarticletitle{A contextual-bandit approach to personalized news
  article recommendation}. In \bibinfo{booktitle}{\emph{Proceedings of the 19th
  international conference on World wide web}}. ACM, \bibinfo{pages}{661--670}.
\newblock


\bibitem[\protect\citeauthoryear{Li, Chu, Langford, and Wang}{Li
  et~al\mbox{.}}{2011}]%
        {li2011unbiased}
\bibfield{author}{\bibinfo{person}{Lihong Li}, \bibinfo{person}{Wei Chu},
  \bibinfo{person}{John Langford}, {and} \bibinfo{person}{Xuanhui Wang}.}
  \bibinfo{year}{2011}\natexlab{}.
\newblock \showarticletitle{Unbiased offline evaluation of
  contextual-bandit-based news article recommendation algorithms}. In
  \bibinfo{booktitle}{\emph{Proceedings of the fourth ACM international
  conference on Web search and data mining}}. ACM, \bibinfo{pages}{297--306}.
\newblock


\bibitem[\protect\citeauthoryear{Li, Karatzoglou, and Gentile}{Li
  et~al\mbox{.}}{2016}]%
        {li2016collaborative}
\bibfield{author}{\bibinfo{person}{Shuai Li}, \bibinfo{person}{Alexandros
  Karatzoglou}, {and} \bibinfo{person}{Claudio Gentile}.}
  \bibinfo{year}{2016}\natexlab{}.
\newblock \showarticletitle{Collaborative filtering bandits}. In
  \bibinfo{booktitle}{\emph{Proceedings of the 39th International ACM SIGIR
  conference on Research and Development in Information Retrieval}}. ACM,
  \bibinfo{pages}{539--548}.
\newblock


\bibitem[\protect\citeauthoryear{Sun and Zhang}{Sun and Zhang}{2018}]%
        {sun2018conversational}
\bibfield{author}{\bibinfo{person}{Yueming Sun} {and} \bibinfo{person}{Yi
  Zhang}.} \bibinfo{year}{2018}\natexlab{}.
\newblock \showarticletitle{Conversational Recommender System}.
\newblock \bibinfo{journal}{\emph{arXiv preprint arXiv:1806.03277}}
  (\bibinfo{year}{2018}).
\newblock


\bibitem[\protect\citeauthoryear{Wang, Wu, and Wang}{Wang
  et~al\mbox{.}}{2016}]%
        {wang2016learning}
\bibfield{author}{\bibinfo{person}{Huazheng Wang}, \bibinfo{person}{Qingyun
  Wu}, {and} \bibinfo{person}{Hongning Wang}.} \bibinfo{year}{2016}\natexlab{}.
\newblock \showarticletitle{Learning hidden features for contextual bandits}.
  In \bibinfo{booktitle}{\emph{Proceedings of the 25th ACM International on
  Conference on Information and Knowledge Management}}. ACM,
  \bibinfo{pages}{1633--1642}.
\newblock


\bibitem[\protect\citeauthoryear{Wu, Wang, Gu, and Wang}{Wu
  et~al\mbox{.}}{2016}]%
        {wu2016contextual}
\bibfield{author}{\bibinfo{person}{Qingyun Wu}, \bibinfo{person}{Huazheng
  Wang}, \bibinfo{person}{Quanquan Gu}, {and} \bibinfo{person}{Hongning Wang}.}
  \bibinfo{year}{2016}\natexlab{}.
\newblock \showarticletitle{Contextual bandits in a collaborative environment}.
  In \bibinfo{booktitle}{\emph{Proceedings of the 39th International ACM SIGIR
  conference on Research and Development in Information Retrieval}}. ACM,
  \bibinfo{pages}{529--538}.
\newblock


\bibitem[\protect\citeauthoryear{Yu, Shen, and Jin}{Yu et~al\mbox{.}}{2019}]%
        {yu2019visual}
\bibfield{author}{\bibinfo{person}{Tong Yu}, \bibinfo{person}{Yilin Shen},
  {and} \bibinfo{person}{Hongxia Jin}.} \bibinfo{year}{2019}\natexlab{}.
\newblock \showarticletitle{An Visual Dialog Augmented Interactive Recommender
  System}. In \bibinfo{booktitle}{\emph{Proceedings of the 25th ACM SIGKDD
  International Conference on Knowledge Discovery \& Data Mining}}. ACM,
  \bibinfo{pages}{157--165}.
\newblock


\bibitem[\protect\citeauthoryear{Yue, Hong, and Guestrin}{Yue
  et~al\mbox{.}}{2012}]%
        {yue2012hierarchical}
\bibfield{author}{\bibinfo{person}{Yisong Yue}, \bibinfo{person}{Sue~Ann Hong},
  {and} \bibinfo{person}{Carlos Guestrin}.} \bibinfo{year}{2012}\natexlab{}.
\newblock \showarticletitle{Hierarchical exploration for accelerating
  contextual bandits}.
\newblock \bibinfo{journal}{\emph{arXiv preprint arXiv:1206.6454}}
  (\bibinfo{year}{2012}).
\newblock


\bibitem[\protect\citeauthoryear{Zeng, Wang, Mokhtari, and Li}{Zeng
  et~al\mbox{.}}{2016}]%
        {zeng2016online}
\bibfield{author}{\bibinfo{person}{Chunqiu Zeng}, \bibinfo{person}{Qing Wang},
  \bibinfo{person}{Shekoofeh Mokhtari}, {and} \bibinfo{person}{Tao Li}.}
  \bibinfo{year}{2016}\natexlab{}.
\newblock \showarticletitle{Online context-aware recommendation with time
  varying multi-armed bandit}. In \bibinfo{booktitle}{\emph{Proceedings of the
  22nd ACM SIGKDD International Conference on Knowledge Discovery and Data
  Mining}}. ACM, \bibinfo{pages}{2025--2034}.
\newblock


\bibitem[\protect\citeauthoryear{Zhang, Chen, Ai, Yang, and Croft}{Zhang
  et~al\mbox{.}}{2018}]%
        {zhang2018towards}
\bibfield{author}{\bibinfo{person}{Yongfeng Zhang}, \bibinfo{person}{Xu Chen},
  \bibinfo{person}{Qingyao Ai}, \bibinfo{person}{Liu Yang}, {and}
  \bibinfo{person}{W~Bruce Croft}.} \bibinfo{year}{2018}\natexlab{}.
\newblock \showarticletitle{Towards conversational search and recommendation:
  System ask, user respond}. In \bibinfo{booktitle}{\emph{Proceedings of
  CIKM}}. ACM, \bibinfo{pages}{177--186}.
\newblock


\end{thebibliography}

\end{document}